%% file: ijcai18-crash.tex
\documentclass[a4paper]{article}
\usepackage[utf8]{inputenc}
\usepackage{a4wide}
\usepackage{enumerate}
\usepackage{float}
\usepackage[normalem]{ulem}
\usepackage{amsmath}
\usepackage{textcomp}
\usepackage{amssymb}
\usepackage{hyperref}
\usepackage{amsthm} 
\usepackage{my-theorems} 
\usepackage{comment} 
\usepackage[]{xcolor} 
\usepackage[final,notref,notcite,color]{showkeys}
\usepackage{framed}  
\usepackage{times} 
\usepackage{AM-style}
\usepackage{thmtools}

\newcommand{\dblnc}{\ensuremath{\mathop{\parallel\!\sim}}} 
 \definecolor{shadecolor}{HTML}{EEF0E7} \definecolor{framecolor}{HTML}{892A54} \definecolor{labelkey}{HTML}{892A54} 
\title{Relevance in Structured Argumentation\makeatletter{\renewcommand*{\@m}{}
    \footnotetext{The paper has been published in the proceedings of IJCAI 2018, main track~\cite{BorStr18}.}
\makeatother}}

 \author{
AnneMarie Borg {\normalfont and} 
Christian Stra{\ss}er,
\\ 
Ruhr-University Bochum, Germany \\
annemarie.borg@rub.de,
christian.strasser@rub.de
}
\date{}

\usepackage[normalem]{ulem}

\begin{document}

\maketitle
\input{00-abstract}

\newenvironment{lbox}{%
    \def\FrameCommand{\fboxrule=\FrameRule\fboxsep=0.6\FrameSep \fcolorbox{framecolor}{shadecolor}}%
    \MakeFramed {\advance\hsize-\width \FrameRestore}}%
    {\endMakeFramed}
\renewcommand*\showkeyslabelformat[1]{%
\fbox{\normalfont\tiny#1}}

\input{01-intro}

\input{02-general-setting}

\input{03-relevance-properties}

\input{04-structured-arg}

\input{meta-theory-flat}

\section{Conclusion}
\label{sec:Conclusion}


In this paper we investigated the robustness of systems of structured argumentation under the addition of irrelevant information. To this end we introduced a simple and easily accessible account of structured argumentation in which meta-theoretic properties can be studied conveniently while systems from the literature can be represented. We studied the properties  Non-Interference, Crash Resistance, and Cumulativity. In future work we plan to incorporate priorities (e.g., \cite{ABS18prio,CyTo16,ModPra12}) and to extend our study to other meta-theoretic properties, such as other properties of non-monotonic inference (\cite{KLM90}) and rationality postulates (\cite{CAm2007}). 







\paragraph{Acknowledgements}
\label{sec:Acknowledgements}

The authors are supported by the Alexander von Humboldt Foundation and the German Ministry for Education and Research. The first author is also supported by the Israel Science Foundation (grant 817/15). 

\bibliographystyle{plain}
\bibliography{thesis.bib}

\end{document}

%% file: 00-abstract.tex
\begin{abstract}
 We study properties related to relevance in non-monotonic consequence relations obtained by systems of structured argumentation. Relevance desiderata concern the robustness of a consequence relation under the addition of irrelevant information. For an account of what (ir)relevance amounts to we use syntactic and semantic considerations. Syntactic criteria have been proposed in the domain of relevance logic and were recently used in argumentation theory under the names of non-interference and crash-resistance. The basic idea is that the conclusions of a given argumentative theory should be robust under adding information that shares no propositional variables with the original database. Some semantic relevance criteria are known from non-monotonic logic. For instance, cautious monotony states that if we obtain certain conclusions from an argumentation theory, we may expect to still obtain the same conclusions if we add some of them to the given database. In this paper we investigate properties of structured argumentation systems that warrant relevance desiderata. 

\end{abstract}


%% file: 01-intro.tex
\section{Introduction}
\label{sec:orgb150017}

In this paper we investigate conditions under which the non-monotonic consequence relation of a given structured argumentation system is robust when irrelevant information is added or removed. 
 Relevance can hereby be understood in two ways. First, syntactically as information that shares propositional variables with the information at hand. Second, semantically, as information that for some reason should not be considered to have defeating power over previously accepted arguments. 

Structured argumentation has been studied in various settings such as ASPIC~\cite{ModPra14,Pra10}, ABA~\cite{Bondarenko1997,Toni14}, and logic-based argumentation~\cite{ArStr15argcomp,BesHun01,BesHun14}. These frameworks share the underlying idea that arguments are to have a logical structure and attacks between them are at least partially determined by logical considerations. Although investigations into translations between these frameworks have been intensified  recently~\cite{HeyStr16}, the frameworks are in various aspects difficult to compare and results obtained in one do not easily transfer to others. For this reason, we decided in this paper to study relevance-related properties for structured argumentation on the basis of a simple framework for structured argumentation that allows us, on the one hand, to abstract away from particularities of the systems from the literature and, on the other hand, to translate these frameworks easily. The framework is simple in that arguments are premise-conclusion pairs \((\Gamma,\gamma)\) obtained from a given consequence relation and it only allows for one type of attack (attacks in premises). The obtained simplicity makes studying meta-theory technically straight-forward and the availability of the translations makes results easily transferable. 

The paper is structured as follows. In Section~\ref{sec:GeneralSetting} we introduce our general setting for structured argumentation. In Section~\ref{sec:RelevanceProperties} we define the basic relevance-related properties that we will investigate in this paper. In Section~\ref{sec:org98c77f8} we show how many of the most common systems of structured argumentation can be represented in our setting. In Section~\ref{sec:org530aa41} we prove our main results. We conclude in Section~\ref{sec:Conclusion}.


%% file: 02-general-setting.tex
\section{General Setting}
\label{sec:org41edb80}
\label{sec:GeneralSetting}

In the following we work with a simple setting for structured argumentation. It is abstract in the sense that it allows for instantiations that are adequate representations of many of the available systems of structured argumentation such as logic-based argumentation, ASPIC, ABA, etc.\ (see Section~\ref{sec:org98c77f8}). In this contribution we restrict ourselves to non-prioritized settings.

We suppose to have available a formal language \(\mathcal{L}\) (we denote the set of well-formed formulas over \(\mathcal{L}\) also by \(\mathcal{L}\)) and a relation \({\vdash} \subseteq \wp_{\sf fin}(\mathcal{L}) \times \mathcal{L}\) (where \(\wp_{\sf fin}\) denotes the set of \emph{finite} subsets) which we will refer to as the \emph{deducability relation}. We do not suppose any of the usual Tarskian properties in what follows (reflexivity, transitivity, and monotonicity). 

\begin{definition}[$\mathit{Arg}_{\vdash}(\cdot)$]
\label{definition:arg}
Given a set of formulas \(\mathcal{S} \subseteq \mathcal{L}\) we denote by \(\mathit{Arg}_{\vdash}(\mathcal{S})\) the set of \(\mathcal{S}\)-based arguments: \((\Gamma, \gamma) \in \mathit{Arg}_{\vdash}(\mathcal{S})\) iff \(\Gamma \vdash \gamma\) for $\Gamma\subseteq\mathcal{S}$. Given \(a = (\Gamma, \gamma) \in \mathit{Arg}_{\vdash}(\mathcal{S})\), \(\mathsf{Conc}(a) = \gamma\) and \(\mathsf{Supp}(a) = \Gamma\).
\end{definition}

To accommodate argumentative attacks we suppose to have two functions: a contrariness function \(\overline{\cdot}: \mathcal{L} \rightarrow \wp(\mathcal{L})\) that associates each formula with a set of conflicting formulas and a function \(\widehat{\cdot}: \wp_{\sf fin}(\mathcal{L}) \setminus \emptyset \rightarrow \wp_{\sf fin}(\mathcal{L})\) that associates support sets with sets of formulas in which they can be attacked. 

\begin{remark}
\label{rem:widehat}
Often \(\widehat{\cdot}\) will simply be the identity function, although another option is, e.g., \(\widehat{\Gamma} = \{ \bigwedge \Gamma' \mid \emptyset \neq \Gamma' \subseteq \Gamma\}\). 
\end{remark}


\begin{definition}[$\mathcal{AF}_{\vdash}$]
An \emph{(argumentation) setting} is a triple \(\mathcal{AF}_{\vdash} = (\vdash, \overline{\cdot}, \widehat{\cdot})\). A setting based on \(\mathcal{S} \subseteq \mathcal{L}\) is given by the quadruple \(\mathcal{AF}_{\vdash}(\mathcal{S}) = (\mathcal{S}, \vdash, \overline{\cdot}, \widehat{\cdot})\).
\end{definition}


\begin{example}
\label{xmpl:CL:id}
A simple example of a setting is \(\mathcal{AF}_{\vdash_{\sf CL}}^{\sf pdef} = (\vdash_{\sf CL}, \overline{\cdot}, \mathsf{id})\) where \(\vdash_{\sf CL}\) is the deducability relation of classical propositional logic and \(\overline{\phi}= \{\neg \phi\}\). 
\end{example}

\begin{example}
\label{xmpl:CL:def}
Another example is the setting \(\mathcal{AF}_{\vdash_{\sf CL}}^{\sf def} = (\vdash_{\sf CL}, \overline{\cdot}, \widehat{\cdot})\) where \(\overline{\phi} = \{\neg \phi \}\) and \(\widehat{\Gamma} = \{ \bigwedge \Delta \mid \emptyset \neq \Delta \subseteq \Gamma\}\).
\end{example}

\begin{definition}[Attacks]
Given a setting \(\mathcal{AF}_{\vdash}(\mathcal{S})\), where \(a = (\Gamma,\gamma) \in \mathit{Arg}_{\vdash}(\mathcal{S})\) and \(b = (\Gamma', \gamma') \in \mathit{Arg}_{\vdash}(\mathcal{S})\), \(a\) \emph{attacks} \(b\) (in \(\phi\)) iff there is a \(\phi \in \widehat{\Gamma'}\) for which \(\gamma \in \overline{\phi}\).
\end{definition}
Our attack form is sometimes called premise-attack~\cite{Pra10} or directed undercut~\cite{BesHun14}. In Section~\ref{sec:org98c77f8} we will show that by adjusting \(\overline{\cdot}\) and \(\widehat{\cdot}\) adequately we are able to accommodate many other attack forms defined in the literature.


\begin{definition}[Attack Diagram]
\label{definition:att:dia}
Given a setting \(\mathcal{AF}_{\vdash}(\mathcal{S})\), its \emph{attack diagram} is the directed graph with the set of nodes \(\mathit{Arg}_{\vdash}(\mathcal{S})\) and edges between \(a\) and \(b\) iff \(a\) attacks \(b\). 
\end{definition}

%
%

%
\begin{definition}[Dung Semantics, \cite{Dung95}]
\label{definition:dung}
Where \(\mathcal{AF}_{\vdash}(\mathcal{S})\) is a setting and \(\mathcal{A} \subseteq \mathit{Arg}_{\vdash}(\mathcal{S})\) we define:
\(\mathcal{A}\) is \emph{conflict-free} iff there are no \(a,b \in \mathcal{A}\) such that \(a\) attacks \(b\).
\(\mathcal{A}\) \emph{defends} \(a \in \mathit{Arg}_{\vdash}(\mathcal{S})\) iff for each attacker \(b \in \mathit{Arg}_{\vdash}(\mathcal{S})\) of \(a\) there is a \(c \in \mathcal{A}\) that attacks \(b\).
\(\mathcal{A}\) is \emph{admissible} iff it is conflict-free and it defends every \(a \in \mathcal{A}\).
\(\mathcal{A}\) is \emph{complete} iff it is admissible and it contains every \(a \in \mathit{Arg}_{\vdash}(\mathcal{S})\) it defends.
\(\mathcal{A}\) is \emph{preferred} iff it is $\subseteq$-maximal complete.
\(\mathcal{A}\) is \emph{grounded} iff it is $\subseteq$-minimal complete. $\mathcal{A}$ is \emph{stable} iff it is admissible and for all $a \in \mathit{Arg}_{\vdash}(\mathcal{S}) \setminus \mathcal{A}$ there is a $b \in \mathcal{A}$ that attacks $a$.

We denote the set of all admissible [complete, preferred, stable] sets $\mathcal{A}$ (also called ``extensions'') by \(\mathsf{Adm}(\mathcal{AF}_{\vdash}(\mathcal{S}))\) [\(\mathsf{Cmp}(\mathcal{AF}_{\vdash}(\mathcal{S})), \mathsf{Prf}(\mathcal{AF}_{\vdash}(\mathcal{S})), \mathsf{Stb}(\mathcal{AF}_{\vdash}(\mathcal{S}))\)] and the grounded set by \(\mathsf{Grd}(\mathcal{AF}_{\vdash}(\mathcal{S}))\). 
\end{definition}


\begin{definition}[Consequence Relations]
\label{definition:nc}
Where \(\mathsf{Sem} \in \{ \mathsf{Grd}, \mathsf{Prf}, \mathsf{Stb}\}\), and given a setting \(\mathcal{AF}_{\vdash}\) we define: 
\begin{itemize}
\item \(\mathcal{S} \nc_{\cap\mathsf{Sem}}^{\mathcal{AF}_{\vdash}} \phi\) iff there is an \(a\in\bigcap \mathsf{Sem}(\mathcal{AF}_{\vdash}(\mathcal{S}))\) with \(\mathsf{Conc}(a) = \phi\);
\item \(\mathcal{S} \nc_{\Cap\mathsf{Sem}}^{\mathcal{AF}_{\vdash}} \phi\) iff for all \(\mathcal{A} \in \mathsf{Sem}(\mathcal{AF}_{\vdash}(\mathcal{S}))\) there is an \(a \in \mathcal{A}\) with \(\mathsf{Conc}(a) = \phi\);
\item \(\mathcal{S} \nc_{\cup\mathsf{Sem}}^{\mathcal{AF}_{\vdash}} \phi\) iff for some \(\mathcal{A} \in \mathsf{Sem}(\mathcal{AF}_{\vdash}(\mathcal{S}))\) there is an \(a \in \mathcal{A}\) with \(\mathsf{Conc}(a) = \phi\);
\end{itemize} 
\end{definition}
Where the setting \(\mathcal{AF}_{\vdash}\) is clear from the context we will simply write \(\nc_{\star {\sf Sem}}\) for \(\star\in\{\cap,\Cap,\cup\}\) to avoid clutter.



%% file: 03-relevance-properties.tex
\section{The Relevance Properties}
\label{sec:orgfeba745}
\label{sec:RelevanceProperties}


  In this section we consider two kinds of relevance. The first, syntactic relevance, is concerned with the information at hand. Known from relevance logics~\cite{AnBel75,DuRes02}, the intuitive idea is that a conclusion should only follow from a set of premises, when the conclusion is relevant. This is implemented by requiring that a formula can only be entailed by a premise set, when the former shares propositional variables with the latter. The second, semantic relevance, is concerned with the addition of information. Known from non-monotonic logic~\cite{KLM90}, the intuitive idea is that the set of consequences from a set of premises should not change if some of these consequences would be added to the premises.

\subsection{Syntactic Relevance}
\label{sec:org0eea835}

A syntactical relevance property that has been proposed in the context of structured argumentation is non-interference~\cite{CaCaDu11}. Let us call two sets of formulas syntactically disjoint if no atom that occurs in a formula in \(\mathcal{S}_1\) also occurs in a formula in \(\mathcal{S}_2\) and vice versa: so \(\mathsf{Atoms}(\mathcal{S}_1) \cap \mathsf{Atoms}(\mathcal{S}_2) = \emptyset\) where \(\mathsf{Atoms}(\mathcal{S})\) is the set of atoms occurring in formulas in \(\mathcal{S}\). In such cases we write: \(\mathcal{S}_1 \mid \mathcal{S}_2\).

\begin{definition}[Non-Interference, \cite{CaCaDu11}]
  \label{def:non-interference}
\(\nc \subseteq \wp(\mathcal{L}) \times \mathcal{L}\) satisfies \emph{Non-Interference} iff for all \(S_1 \cup \{\phi\} \cup \mathcal{S}_2 \subseteq \mathcal{L}\) for which \((\mathcal{S}_1 \cup \{\phi\}) \mid \mathcal{S}_2\) we have:\footnote{A similar property is Basic Relevance \cite[Definition~3.1]{Avr14}.}  $\mathcal{S}_1 \nc \phi \mbox{ iff } \mathcal{S}_1 \cup \mathcal{S}_2 \nc \phi.$
\end{definition}




\begin{definition}[Contamination, \cite{CaCaDu11}]
  Let \(\nc \subseteq \wp(\mathcal{L}) \times \mathcal{L}\) be a consequence relation. A set \(\mathcal{S}\subseteq\mathcal{L}\), such that $\mathsf{Atoms}(\mathcal{S})\subset\mathsf{Atoms}(\mathcal{L})$, is called \emph{contaminating} (with respect to $\nc$), if for any set of formulas $\mathcal{S}'\subseteq\mathcal{L}$ such that \(\mathcal{S} \mid\mathcal{S}'\) and for every $\phi\in\mathcal{L}$, it holds that $\mathcal{S}\:\nc\:\phi$ if and only if $\mathcal{S}\cup\mathcal{S}'\:\nc\:\phi$. 
\end{definition}

Consequence relations that are non-trivial and satisfy Non-Interference also satisfy Crash-Resistance:\footnote{$\nc$ is non-trivial if there are always two sets of formulas with the same atoms but different conclusions (see \cite{CaCaDu11}).}

\begin{definition}[Crash-Resistance, \cite{CaCaDu11}]
 A consequence relation \(\nc \subseteq \wp(\mathcal{L}) \times \mathcal{L}\) satisfies \emph{Crash-Resistance} iff there is no set $\mathcal{S}\subseteq\mathcal{L}$ that is contaminating with respect to $\nc$.
\end{definition}

Given a setting \(\mathcal{AF}_{\vdash}\), a natural question is whether Non-Interference is a property that gets inherited on the level of non-monotonic inference \(\nc_{\sf sem}\) from \(\vdash\): we will show below that in case \(\vdash\) satisfies Non-Interference so does \(\nc_{\sf sem}\). In fact, the following less requiring criterion is sufficient:

\begin{definition}[Pre-Relevance]
\({\vdash} \subseteq \wp(\mathcal{L}) \times \mathcal{L}\) satisfies \emph{Pre-Relevance} iff for all \(\mathcal{S}_1 \cup \{\phi\} \cup \mathcal{S}_2 \subseteq \mathcal{L}\) for which \(\mathcal{S}_1 \cup \{\phi\} \mid \mathcal{S}_2\): if \(\mathcal{S}_1 \cup \mathcal{S}_2 \vdash \phi\) then there is a \(\mathcal{S}_1' \subseteq \mathcal{S}_1\) such that \(\mathcal{S}_1' \vdash \phi\).
\end{definition}

When considering attacks we need to extend the notion of Pre-Relevance by taking into account $\widehat{\cdot}$ and $\overline{\cdot}$. We first define:

\begin{definition}[Prime Settings]
\label{def:prime:setting}
  A setting \((\vdash, \overline{\cdot}, \widehat{\cdot})\) is \emph{prime} iff for all sets of at\-oms \(\mathcal{A}_1\) and \(\mathcal{A}_2\) in \(\mathcal{L}\) for which \(\mathcal{A}_1 \mid \mathcal{A}_2\), for all \(\mathcal{S}_1, \mathcal{T}_1, \mathcal{S}_2, \mathcal{T}_2 \in \wp_{\sf fin}(\mathcal{L})\) for which \(\mathsf{Atoms}(\mathcal{S}_1), \allowbreak \mathsf{Atoms}(\mathcal{T}_1) \subseteq \mathcal{A}_1\) and \(\mathsf{Atoms}(\mathcal{S}_2), \mathsf{Atoms}(\mathcal{T}_2) \subseteq \mathcal{A}_2\), and for all \(\phi\) and \(\psi\) such that \(\psi \in \overline{\phi}\) and \(\phi \in \widehat{\mathcal{T}_1 \cup \mathcal{T}_2}\), we have: \\
  if \(\mathcal{S}_1 \cup \mathcal{S}_2 \vdash \psi\) then there are \(i \in \{1,2\}\), \(\mathcal{S}_i' \subseteq \mathcal{S}_i\), \(\phi_i \in \widehat{\mathcal{T}_i}\) and \(\psi_i \in \overline{\phi_i}\) for which \(\mathcal{S}_i' \vdash \psi_i\).
\end{definition}

\begin{definition}[Pre-Relevant Settings]
  \label{def:pre-rel-AF}
  A setting \(\mathcal{AF} = ({\vdash,} \overline{\cdot}, \widehat{\cdot})\) is \emph{Pre-Relevant} iff (i) \(\vdash\) is Pre-Relevant, (ii) \(\mathcal{AF}\) is prime, and (iii) \(\widehat{\cdot}\) is \(\subseteq\)-monotonic (i.e., for all $\Delta, \Delta' \in \wp_{\mathsf{fin}}(\mathcal{L}), \widehat{\Delta} \subseteq \widehat{\Delta \cup \Delta'}$).
\end{definition}

\begin{example}
  Note that \(\widehat{\cdot}: \Delta \mapsto \Delta\) (see Example~\ref{xmpl:CL:id}) and \(\widehat{\cdot}: \Delta \mapsto \{\bigwedge \Delta' \mid \emptyset \subset \Delta' \subseteq \Delta\}\) (see Example~\ref{xmpl:CL:def}) are both \(\subseteq\)-monotonic.
\end{example}

\begin{fact}
  Where \(\widehat{\cdot} = \mathsf{id}(\cdot)\) (see Example~\ref{xmpl:CL:id}) and \(\overline{\gamma} = \{\neg \gamma\}\), the Pre-Relevance of \((\vdash, \overline{\cdot}, \widehat{\cdot})\) follows from the Pre-Relevance of \(\vdash\).
\end{fact}
\begin{proof}
  Items (i) and (iii) are trivial. For Item (ii) suppose that \(\mathcal{S}_1 \cup \mathcal{S}_2 \vdash \psi\), where \(\psi = \neg \phi\) and \(\phi \in \widehat{\mathcal{T}_1 \cup \mathcal{T}_2} = \mathcal{T}_1 \cup \mathcal{T}_2\) and where \(\mathcal{S}_1, \mathcal{S}_2, \mathcal{T}_1, \mathcal{T}_2\) are as in Definition~\ref{def:prime:setting}. Thus, there is an \(i \in \{1,2\}\) s.t.\ \(\phi \in \mathcal{T}_i\). Thus, \(\mathsf{Atoms}(\psi) \subseteq \mathcal{A}_i\). By the Pre-Relevance of \(\vdash\), there is an \(\mathcal{S}_i' \subseteq \mathcal{S}_i\) for which \(\mathcal{S}_i' \vdash \psi\).
\end{proof}

\begin{fact}
  Where \(\widehat{\Delta} = \{\bigwedge \Delta' \mid \emptyset \subset \Delta' \subseteq \Delta\}\) (see Example~\ref{xmpl:CL:def}), \(\overline{\gamma} = \{\neg \gamma\}\) and \(\vdash\) is contrapositable (i.e., \(\mathcal{S} \vdash \neg \bigwedge (\Delta \cup \Delta')\) implies \(\mathcal{S} \cup \Delta' \vdash \neg \bigwedge \Delta\)), the Pre-Relevance of \((\vdash, \overline{\cdot}, \widehat{\cdot})\) follows from the Pre-Relevance of \(\vdash\).
\end{fact}



In Section~\ref{sec:org31d680c} we will show that:
\begin{restatable}{theorem}{noninter}
\label{thm:non-inter}
If \(\mathcal{AF}_{\vdash}\) satisfies Pre-Relevance then $\nc_{\cap\sem}^{\calAF_\vdash}$, \(\nc_{\Cap{\sf sem}}^{\mathcal{AF}_{\vdash}}\)and \(\nc_{\cup{\sf sem}}^{\calAF_{\vdash}}\) satisfy Non-Inter\-ference for each \(\sem\in\{\grd,\cmp,\prf\}\).
\end{restatable}

\begin{example}
\label{xmpl:RM}
We take the setting \(\mathcal{AF}_{\vdash_{\sf RM}} = (\vdash_{\sf RM}, \overline{\cdot}, \mathsf{id})\), where \(\vdash_{\sf RM}\) is the consequence relation of the semi-relevance logic \(\mathsf{RM}\) and $\overline{\cdot}: \phi \mapsto \{\neg \phi\}$. \(\vdash_{\sf RM}\) satisfies Pre-Relevance (see \cite[Proposition~6.5]{Avr16}) and thus \(\mathcal{AF}_{\vdash_{\sf RM}}\) satisfies Non-Inference and Crash-Resistance. Similar for other relevance logics.
\end{example}

\begin{example}
\label{xmpl:CL:top}
Although \(\vdash_{\sf CL}\) does not satisfy Pre-Relevance, \(\vdash_{\sf CL}^{\top}\) does, where \(\vdash_{\sf CL}^{\top}\) is the restriction of \(\vdash_{\sf CL}\) to pairs \((\Gamma,\gamma)\) for which \(\nvdash_{\sf CL} \neg \bigwedge \Gamma\). Hence, \(\mathcal{AF}_{\vdash_{\sf CL}^{\top}}  = \bigl(\vdash_{\sf CL}^{\top}, \overline{\cdot}, \mathsf{id} \bigr)\) where $\overline{\cdot}: \phi \mapsto \{\neg \phi\}$ satisfies Non-Interference. In~\cite{WuPod14} such a restriction is applied in the context of ASPIC.
\end{example}

\begin{example}
\label{xmpl:mcs}
Recently paraconsistent logics based on maximal consistent subsets~\cite{GroPra16} have been used in the context of structured argumentation. Let \(\Gamma \vdash_{\sf mcs}^{\Cap} \phi\) [\(\Gamma \vdash_{\sf mcs}^{\cup} \phi\)] iff for all [some] maximal consistent subsets \(\Gamma'\) of \(\Gamma\), \(\Gamma' \vdash_{\sf CL} \phi\). (\(\Gamma' \subseteq \Gamma\) is a maximal consistent subset of \(\Gamma\) if it is consistent and there are no consistent \(\Gamma'' \subseteq \Gamma\) such that \(\Gamma' \subset \Gamma''\).) Such consequence relations satisfy Pre-Relevance and thus, argumentative settings based on them satisfy Non-Interference.
\end{example}

A refinement of Theorem \ref{thm:non-inter} is given in Corollary \ref{cor:non-inter} below.

\begin{definition}
Given a setting \((\vdash, \overline{\cdot}, \widehat{\cdot})\) let \(\vdash^{\emptyset}\) be the restriction of \(\vdash\) to pairs \((\Gamma,\gamma)\) for which there is no \((\emptyset, \delta) \in {\vdash}\) such that \(\delta \in \overline{\psi}\) for some \(\psi \in \widehat{\Gamma}\).
\end{definition}

Since arguments with empty supports have no attackers we have:
\begin{lemma}
Where \(\sem\in\{\grd,\cmp,\prf\}\) and \(\mathcal{S} \subseteq \mathcal{L}\),  \[\textstyle \mathcal{S}\: \nc_{\star{\sf sem}}^{(\vdash, \overline{\cdot}, \widehat{\cdot})}\: \phi \mbox{ \ iff \ } \mathcal{S}\: \nc_{\star{\sf sem}}^{(\vdash^{\emptyset}, \overline{\cdot}, \widehat{\cdot})}\: \phi \qquad\qquad \text{ where }\star\in\{\cap,\Cap,\cup\}.\]
\end{lemma}

\begin{corollary}
  \label{cor:non-inter}
If \(\mathcal{AF}_{\vdash^{\emptyset}}\) satisfies Pre-Relevance then $\nc_{\cap\sem}^{\calAF_\vdash}$, \(\nc_{\Cap{\sf sem}}^{\mathcal{AF}_{\vdash}}\) and \(\nc_{\cup{\sf sem}}^{\AF_\vdash}\) sat\-is\-fy Non-Interference for each \(\sem\in\{\grd,\cmp,\prf\}\).
\end{corollary}

We illustrate the latter point with an example. 

\begin{example}
Also the setting \(\mathcal{AF}_{\vdash_{\sf CL}}^{\sf def}\) in Example~\ref{xmpl:CL:def} satisfies Non-Interference. Note for this that \({\vdash_{\sf CL}^{\emptyset} }= {\vdash_{\sf CL}^{\top}}\) (where the latter is defined as in Example~\ref{xmpl:CL:top}) in the context of \(\mathcal{AF}_{\vdash_{\sf CL}}^{\sf def}\).
\end{example}

\begin{remark}
  Odd cycles of arguments (e.g., for arguments $a$, $b$, and $c$, $a$ attacks $b$ attacks $c$ attacks $a$), cause the absence of stable semantics. In such a case the consequence relation for stable semantics would violate non-interference. Thus, the above results do not hold for stable semantics. Therefore, we will not always consider stable semantics in the remainder of the paper.  
\end{remark}

In the following sections we will relate these results to  systems of structured argumentation from the literature.

\subsection{Semantic Relevance}
\label{sec:orgf079c13}


We now turn to properties concerned with information that should not have defeating power over previously accepted arguments. For this, we study a criterion known from non-monotonic logic, namely Cumulativity. Intuitively, Cumulativity states that adding derivable formulas to the premise set, does not change the set of consequences.


\begin{definition}
  \label{def:consequence:plusphi}
  Given \({\vdash} \subseteq \wp(\mathcal{L}) \times \mathcal{L}\) and \(\phi \in \mathcal{L}\), let \({\vdash^{+\phi}}\) be the transitive closure of \({\vdash} \cup \{(\emptyset, \phi)\}\). Given a setting \(\mathcal{AF}_{\vdash}\) and a semantics \(\sem\in\{\grd,\cmp,\prf\}\) let \(\nc_{\star\sf sem}^{+\phi}\) be an abbreviation of \(\nc_{\star\sf sem}^{\mathcal{AF}_{\vdash^{+\phi}}}\) for $\star\in\{\cap,\Cap,\cup\}$ and $\mathcal{AF}_{\vdash}^{+\phi}$ for $\mathcal{AF}_{\vdash^{+\phi}}$.
\end{definition}


On the level of consequence relations Cumulativity is the following property:
\begin{definition}[Cumulativity]
A setting \(\mathcal{AF}_{\vdash}\) satisfies \emph{Cumulativity} for \(\sem\in\{\grd,\allowbreak\cmp,\allowbreak\prf,\allowbreak\stb\}\) and $\star\in\{\cap,\Cap,\cup\}$, iff, for all \(\mathcal{S} \cup \{\phi, \psi\} \subseteq \mathcal{L}\) such that \(\mathcal{S} \nc_{\sf sem}^{\star} \phi\) we have: 
$\mathcal{S} \nc_{\star\sf sem}^{+\phi} \psi \mbox{ iff } \mathcal{S} \nc_{\sf sem}^{\star} \psi$.
\end{definition}

\begin{definition}[Monotonicity]
  \label{def:monotonicity}
  A setting \(\mathcal{AF}_\vdash\) satisfies \emph{Monotonicity} for \(\sem\in\{\grd,\allowbreak\cmp,\allowbreak\prf,\allowbreak\stb\}\) and \(\star\in\{\cap,\Cap,\cup\}\), iff, for all \(\mathcal{S} \cup \{\phi, \psi\} \subseteq \mathcal{L}\), \(\mathcal{S}\:\nc_{{\sf sem}}^\star\:\psi\) implies \(\mathcal{S}\cup\{\phi\}\:\nc_{{\sf sem}}^\star\:\psi\). 
\end{definition}

On the level of Dung-extensions, Cumulativity is:
\begin{definition}[Extensional Cumulativity]
  A setting \(\mathcal{AF}_{\vdash}\) satisfies \emph{Extensional Cumulativity} for \(\sem\in\{\grd,\cmp,\prf\}\) if and only if for all $\mathcal{S}\cup\{\phi\}\subseteq\mathcal{L}$ such that $\mathcal{S}\:\nc_{\mathsf{sem}}\:\phi$ we have that: \(\exts_\sem(\mathcal{AF}_{\vdash}(\mathcal{S})) = \left\{ \mathcal{E} \cap \mathit{Arg}_{\vdash}(\mathcal{S}) \mid \mathcal{E} \in \exts_\sem(\mathcal{AF}_{\vdash^{+\phi}}(\calS)) \right\}.\)
\end{definition}


We will show, in Section~\ref{sec:org67d6cb1}, that a  setting \(\mathcal{AF}_{\vdash}\) satisfies Cumulativity for grounded semantics if \(\mathcal{AF}_{\vdash}\) is pointed. In the definition of a pointed setting, we will also introduce the notion of Cut. Though known from sequent calculi~\cite{Gen34} in which it represents transitivity, we consider Cut here with respect to $\vdash^{+\phi}$, from Definition~\ref{def:consequence:plusphi}: 

\begin{definition}[Pointed Settings]
\((\vdash, \overline{\cdot}, \widehat{\cdot})\) is \emph{pointed} iff 
\begin{enumerate}
\item for all \(\Gamma, \Delta \in \wp_{\sf fin}(\mathcal{L})\), \(\widehat{\Gamma \cup \Delta} = \widehat{\Gamma} \cup \widehat{\Delta}\) (in this case we say that \(\widehat{\cdot}\) is pointed), and
\item \(\vdash\) satisfies Cut w.r.t.\ $\vdash^{+\phi}$ for any $\phi \in \mathcal{L}$, i.e., for every \(\Gamma \cup \{\gamma\} \subseteq \mathcal{L}\), \(\Gamma \cup \Delta \vdash \gamma\) if \(\Gamma \vdash \phi\) and \(\Delta \vdash^{+\phi} \gamma\).
\end{enumerate}
\end{definition}


\begin{theorem}
  \label{thm:cum}
  Where \(\mathcal{AF}_{\vdash}\) is pointed, \(\mathcal{AF}_{\vdash}\) satisfies Cumulativity and Extensional Cumulativity for grounded semantics.
\end{theorem}

\begin{example}
Any setting \((\vdash, \overline{\cdot}, \mathsf{id})\) is pointed iff \(\vdash\) satisfies Cut. For instance, each of the consequence relations in Examples \ref{xmpl:CL:id} and \ref{xmpl:RM}  satisfies Cut and thus the corresponding settings are pointed and therefore satisfy Cumulativity.
\end{example}

\begin{remark}
  In Section~\ref{sec:org67d6cb1} we show, after the proof, that Theorem~\ref{thm:cum} does not hold for preferred semantics, nor does it hold when $\vdash$ does not satisfy Cut. 
\end{remark}

If we restrict $\vdash$ to consistent sets on the left side, denoted by $\vdash_{\mathsf{con}}$ (see Definition~\ref{df:consistent:AF} below) and if $\vdash$ satisfies Cut and Contraposition (see Definition~\ref{df:contrapositable-setting} below), then the setting $\mathcal{AF}_{\mathsf{con}} =  (\vdash_{\mathsf{con}},$ $\overline{\cdot}, \mathsf{id})$ is cumulative. In more detail:

\begin{definition}
\label{df:contrapositable-setting}
  \((\vdash, \overline{\cdot})\) is \emph{contrapositable} iff for all \(\Theta \in \wp_{\mathsf{fin}}(\mathcal{L})\), if \(\Theta \vdash \gamma'\) where \(\gamma' \in \overline{\gamma}\) then for all \(\sigma \in \Theta\), \((\Theta \cup \{\gamma\}) \setminus \{\sigma\} \vdash \sigma'\) for some \(\sigma' \in \overline{\sigma}\). By extension we call \(\mathcal{AF}_{\vdash} = \langle \vdash, \overline{\cdot}, \widehat{\cdot} \rangle\) contrapositable if \((\vdash, \overline{\cdot})\) is contrapositable.
\end{definition}

\begin{definition}
  \label{df:consistent:AF}  
  Where \(\mathcal{AF}_{\vdash} = \langle \vdash, \overline{\cdot}, \mathsf{id} \rangle\), a set \(\Theta \subseteq \mathcal{S}\) is \emph{\(\mathcal{AF}_{\vdash}(\mathcal{S})\)-inconsistent} iff there is a \(\Theta' \subseteq \Theta\) and a \(\gamma \in \Theta'\) for which  \(\Theta \setminus \{\gamma\} \vdash \gamma'\) where \(\gamma' \in \overline{\gamma}\). \(\Theta\) is \emph{\(\mathcal{AF}_{\vdash}(\mathcal{S})\)-consistent} iff it is not \(\mathcal{AF}_{\vdash}(\mathcal{S})\)-inconsistent.
\end{definition}

\noindent Given $\vdash$, let $\vdash_{\mathsf{con}} = \left\{ (\Gamma, \gamma) \mid \Gamma \vdash \gamma \mbox{ and } \Gamma \mbox{ is } \mathcal{AF}_{\vdash} \mbox{-consistent} \right\}$.


\begin{restatable}{theorem}{CUMcon}
\label{thm:CUM:con}
Where \(\mathcal{AF}_{\vdash} = (\vdash, \overline{\cdot}, \mathsf{id})\) is contrapositable and \(\vdash\) satisfies Cut, \(\mathcal{AF}_{\mathsf{con}} = (\vdash_{\mathsf{con}}, \overline{\cdot}, \mathsf{id})\) is cumulative and extensionally cumulative for \(\sem\in\{\grd,\prf,\stb\}\) and the (weakly) skeptical entailment relation.
\end{restatable}

\begin{example}
  In view of Theorem~\ref{thm:CUM:con}, $\mathcal{AF}_{\vdash_{\mathsf{CL}}^{\top}}$ from Example~\ref{xmpl:CL:top} is cumulative for $\sem\in\{\grd,\allowbreak\stb,\allowbreak\prf\}$.
\end{example}

\begin{restatable}{theorem}{Monotonicity}
\label{thm:monotonicty}
  For argumentation settings \(\mathcal{AF}_{\vdash} = (\vdash, \overline{\cdot}, \mathsf{id})\) that are contrapositable and where $\vdash$ satisfies Cut, \(\mathcal{AF}_{\mathsf{con}} = (\vdash_{\mathsf{con}}, \overline{\cdot}, \mathsf{id})\) is monotonic for \(\sem\in\{\prf,\stb\}\) and the credulous entailment relation.
\end{restatable}

Below we give an example to show that the above theorem does not hold for the skeptical entailment relation. 
\begin{example}
  \label{ex:counter:monotonicskeptical}
  Consider the setting $\AF^{\sf pdef}_{\vdash_{\sf CL}} = (\vdash_{\sf CL}, \overline{\cdot},{\sf id})$, from Example~\ref{xmpl:CL:id} and let $\calS = \{p\}$. Note that $p\vdash_{\sf CL} p$. Moreover, we have that $\mathit{Arg}_{\vdash_{\sf CL}}(\calS) = \{(\Gamma,\phi)\mid \Gamma\subseteq\{p\}\text{ and }\phi\in{\sf CN}(\{p\})\}$, this is also the only extension, for any semantics from Definition~\ref{definition:dung}. Thus $\calS\:\nc^{\AF^{\sf pdef}_{\vdash_{\sf CL}}}_{\star{\sf sem}}\:p$ for $\star\in\{\cap,\Cap,\cup\}$ and ${\sf sem}\in\{{\sf prf}, {\sf stb}\}$. 
  \\
  Now consider $\calS' = \calS\cup\{\neg p\} = \{p,\neg p\}$. Then there are arguments $(p,p), (\neg p, \neg p)\in\textit{Arg}_{\vdash_{\sf CL}}(\calS'))$ and there is no longer just one extension. Therefore, although $\calS'\:\nc^{\AF^{\sf pdef}_{\vdash_{\sf CL}}}_{\cup{\sf sem}}\:p$ and $\calS'\:\nc^{\AF^{\sf pdef}_{\vdash_{\sf CL}}}_{\cup{\sf sem}}\:\neg p$, neither $\calS'\:\nc^{\AF^{\sf pdef}_{\vdash_{\sf CL}}}_{\cap{\sf sem}}\:p$ nor $\calS'\:\nc^{\AF^{\sf pdef}_{\vdash_{\sf CL}}}_{\cap{\sf sem}}\:\neg p$ for ${\sf sem}\in\{{\sf prf}, {\sf stb}\}$. 
\end{example}




%% file: 04-structured-arg.tex
\section{Systems of Structured Argumentation}
\label{sec:org98c77f8}

In this section we take a look at several of the structured argumentation frameworks from the literature and show how they can be represented in our setting. 

\begin{example}[Logic-Based Argumentation]
\label{xmpl-lba}
Logic-based argumentation is closest to our setting from Section~\ref{sec:GeneralSetting}.  Systems can be found in, for instance, \cite{ArStr15argcomp,BesHun01}.\footnote{There are differences between these presentations: while \cite{BesHun01,BesHun14} use classical logic as a core logic, \cite{ArStr15argcomp} allows for any Tarskian logic with an adequate sequent calculus to serve as core logic.  \cite{BesHun01,BesHun14} require the support sets of arguments to be consistent and minimal while \cite{ArStr15argcomp} omit this requirement. In what follows we follow the generalized setting of \cite{ArStr15argcomp}. Consistency and minimality can easily be captured by changing the underlying relation \(\vdash\) (see e.g., Example~\ref{xmpl:CL:top}).} The core logic \(\mathsf{L}\) is a finitary Tarskian logic with an adequate consequence relation \({\vdash} \subseteq \wp_{\sf fin}(\mathcal{L}) \times \mathcal{L}\). Given a set \(\mathcal{S} \subseteq \mathcal{L}\), the set of arguments defined by \(\mathit{Arg}_{\vdash}(\mathcal{S})\) consists of all \((\Gamma, \gamma)\) where \(\Gamma \vdash \gamma\) and $\Gamma\subseteq\mathcal{S}$ just like in Definition~\ref{definition:arg}. Different attack rules have been proposed, such as: $(\Gamma, \gamma)$ attacks $(\Delta, \psi)$ iff \ldots
\begin{enumerate}[{\hspace{-2mm}}]
\item \emph{Defeat (Def)}: ~~ \(\gamma \vdash \neg \bigwedge \Delta' \) for some $\emptyset\neq\Delta'\subseteq\Delta$.

\item \emph{Undercut (Ucut)}: ~~ \(\vdash \gamma \equiv \neg \bigwedge \Delta'\) for some $\emptyset\neq\Delta'\subseteq\Delta$.

\item \emph{Direct Compact Defeat (DiCoDef)}:  \(\gamma = \neg \delta'\) for some \(\delta' \in \Delta\).

\item \emph{Direct Undercut (DiUcut)}: ~~ there is a \(\delta \in \Delta\) s.t.\ \(\vdash\gamma \equiv \neg\delta\).

\item \emph{Direct Defeat (DiDef)}: ~~ there is a \(\delta \in \Delta\) s.t.\ \(\gamma \vdash \neg\delta\).

\end{enumerate}

Dung semantics are defined as usual on top of an attack diagram analogous to Definitions~\ref{definition:att:dia} and \ref{definition:dung}. Consequence relations are defined analogous to Definition \ref{definition:nc}, here denoted by \(\mathcal{S} \dblnc_{\mathsf{sem}}^\star \phi\) for $\star\in\{\cap,\Cap,\cup\}$.

Systems of logic-based argumentation translate rather directly to our setting. We only need to adjust the definitions of \(\overline{\cdot}\) and \(\widehat{\cdot}\) so that we can use our attack definition to simulate the attack definitions above. The following table shows how: \smallbreak

\noindent \begin{tabular}{lcc}
 & \(\overline{\delta}\) & \(\widehat{\Delta}\)\\
\hline
DiCoDef & \(\{\neg \delta\}\) & \(\Delta\)\\
Def & \(\{\neg \delta\}\) & \(\{\bigwedge \Delta' \mid \emptyset \subset \Delta'\subseteq\Delta\}\)\\
DiDef & \(\{\gamma \mid \gamma \vdash \neg \delta\}\) & \(\Delta\)\\
DiUcut & \(\{\gamma \mid \gamma \vdash \neg \delta, \neg \delta \vdash \gamma\}\) & \(\Delta\)\\
Ucut & \(\{\gamma \mid \gamma \vdash \neg \delta,  \neg\delta \vdash \gamma\}\) & \(\{\bigwedge \Delta' \mid \emptyset \subset \Delta'\subseteq\Delta\}\)\\
\end{tabular} \smallbreak

The easy proof concerning the adequacy of our representations is omitted.
\end{example}

\begin{remark}
  \label{rem:lba:cum} 
  The definitions for direct attack forms (DiDef, DiUcut, DiCoDef) all give rise to a pointed $\widehat{\cdot}$ (namely $\mathsf{id}$) in our representation. Thus, combining these attack forms with core logics ${\sf L}$ for which $\vdash_{\mathsf{L}}$ satisfies Cut, we obtain Cumulativity.
\end{remark}

\begin{remark}
  \label{rem:lba:non-inter} 
  By instantiating logic-based argumentation with a core logic that satisfies Pre-Relevance (such as the ones in Examples \ref{xmpl:RM}, \ref{xmpl:CL:top}, \ref{xmpl:mcs}) we obtain Non-Interference.
\end{remark}

\smallskip

\begin{example}[Assumption-Based Argumentation (ABA), \cite{Bondarenko1997}]
\label{xmpl:aba}
Let \(\mathcal{L}\) be a formal language, \(\overline{\cdot}: \mathcal{L} \rightarrow \wp(\mathcal{L})\) a contrariness function, \(Ab \subseteq \mathcal{L}\) a subset of so-called \emph{assumptions}, and \(\mathcal{R}\) be a set of rules of the form \(\phi_1, \dotsc, \phi_n \rightarrow \phi\) where \(\phi_1, \dotsc, \phi_n, \phi \in \mathcal{L}\) and \(\phi \notin Ab\).\footnote{In this paper we restrict ourselves to so-called flat frameworks that satisfy the latter requirement.} There is an \emph{\(\mathcal{R}\)-deduction} from some \(\Delta \subseteq Ab\) to \(\phi\) iff there is a sequence \(\phi_1, \dotsc, \phi_n\) for which \(\Delta = \{\phi_1, \dotsc, \phi_n\} \cap Ab\), \(\phi_n = \phi\) and for each \(1 \le i \le n\), \(\phi_i\) is either in \(\Delta\) or there is a rule \(\phi_{i_1}, \dotsc, \phi_{i_m} \rightarrow \phi_i\) where \(i_1, \dotsc, i_m < i\). Given two sets of assumptions \(\Delta, \Delta' \subseteq Ab\), \(\Delta\) \emph{attacks} \(\Delta'\) iff there is a \(\delta \in \Delta'\) for which there is an \(\mathcal{R}\)-deduction of some \(\psi \in \overline{\delta}\) from some \(\Delta'' \subseteq \Delta\). Subsets of assumptions in \(Ab\) and attacks between them give rise to an attack diagram where nodes are sets of assumptions and arcs are attacks. Dung-style semantics are applied to these graphs: \(\Delta\) is conflict-free if it does not attack itself, \(\Delta\) is admissible if it defends itself, it is complete if it contains all assumptions it defends, it is preferred if it is maximally admissible and stable if it is admissible and attacks every assumption it does not contain. Given a semantics \(\mathsf{sem}\), a consequence relation is given by \((Ab,\mathcal{R}) \:\dblnc_{\Cap\mathsf{sem}}^{\sf aba} \:\phi\) [\((Ab,\calR)\:\dblnc_{\cup\sem}^{\sf aba}\:\phi$ respectively \((Ab,\calR)\:\dblnc_{\cap\sem}^{\sf aba}\:\phi$] iff \(\phi\) is \(\mathcal{R}\)-derivable from all [some respectively the intersection of the] sets of assumptions \(\Delta \subseteq Ab\) that satisfy the requirements of \(\mathsf{sem}\).

In most presentations of ABA, the rules $\mathcal{R}$ are considered domain-specific strict inference rules that are part of a given knowledge base. They may also be obtained from an underlying core logic $\mathsf{L}$ with consequence relation $\vdash_{\mathsf{L}}$ by setting $\phi_1, \dotsc, \phi_n \rightarrow \phi$ iff $\{\phi_1, \dotsc, \phi_n\} \vdash_{\mathsf{L}} \phi$.

We can translate ABA into our setting as follows. Where $\mathcal{R}$ represents domain-specific rules that are part of the knowledge base, we define for \(\Delta \subseteq Ab\) and $\mathcal{R}' \subseteq \mathcal{R}$:
\begin{description}
\item[$(\dagger_{\sf aba})$] \(\Delta \cup \mathcal{R}' \vdash \phi\), iff, there is an \(\mathcal{R}\)-deduction of \(\phi\) from \(\Delta\) making use of the rules in $\mathcal{R}'$ (and only of these).\footnote{\label{fn:prem-rules}For this the language $\mathcal{L}$ underlying the original ABA framework is enriched by $\mathcal{R}$ so that ${\vdash} \subseteq \wp_{\sf fin}(\mathcal{L} \cup \mathcal{R}) \times \mathcal{L}$. This is important to track syntactic relevance.}
\end{description}

Where $\mathcal{R}$ is generated from a given core logic $\mathsf{L}$, we define for \(\Delta \subseteq Ab\):
\begin{description}
\item[$(\ddagger_{\sf aba})$] $\Delta \vdash \phi$, iff, $\Delta \vdash_{\mathsf{L}} \phi$.
\end{description}

In both cases, we use the definition of \(\overline{\cdot}\) from ABA, let $\widehat{\cdot} = \mathsf{id}(\cdot)$. Clearly, in our setting \((\Delta, \delta)\) attacks  \((\Gamma, \gamma)\) iff \(\delta \in \overline{\phi}\) for some \(\phi \in \Gamma\). %
We omit the proof that the setting $\mathcal{AF}_{\vdash}(Ab \cup \mathcal{R})$ [respectively $\mathcal{AF}_{\vdash}(Ab)$] adequately represents the ABA framework based on $Ab$ and $\mathcal{R}$ for $\vdash$ in ($\dagger_{\sf aba}$) [respectively ($\ddagger_{\sf aba}$)] and $\star\in\{\Cap,\cup,\cap\}$ so that $(Ab, \mathcal{R})\: \dblnc_{\star{\sf sem}}^{\sf aba}\: \phi$ iff $Ab \cup \mathcal{R} \:\nc_{\star{\sf sem}}^{\mathcal{AF}_{\vdash}}\: \phi$ [respectively $Ab\:\nc_{\star{\sf sem}}^{\mathcal{AF}_{\vdash}} \:\phi$].




\end{example}

\begin{remark}
  \label{rem:aba:non-inter} 
  It is easy to see that for representation ($\dagger_{\sf aba}$) the underlying consequence relation $\vdash$ satisfies Pre-Relevance and if ($\dagger$) $\mathsf{Atoms}(\overline{\phi}) \subseteq \mathsf{Atoms}(\phi)$ for all $\phi \in \mathcal{L}$, we obtain Non-Interference. For the representation $(\ddagger_{\sf aba})$ it depends on the logic $\mathsf{L}$. In case $\vdash_{\mathsf{L}}$ satisfies Pre-Relevance and if ($\dagger$) we obtain Non-Interference.
\end{remark}

\begin{remark}
  \label{rem:aba:cum} 
  Our representation of ABA makes use of the pointed $\widehat{\cdot}$ (namely $\mathsf{id}$) and $\mathcal{R}$-derivability satisfies Cut. Note that $\mathcal{AF}_{\vdash^{+\phi}}(Ab \cup \mathcal{R})$ [resp.\ $\mathcal{AF}_{\vdash^{+\phi}}(Ab)$] adequately represents the ABA framework based on $(Ab, \mathcal{R} \cup \{\rightarrow \phi\})$ for $\vdash$ in ($\dagger_{\sf aba}$) [resp.\ for $\vdash$ in ($\ddagger_{\sf aba}$)]. Thus we obtain Cumulativity.
\end{remark}

\smallskip

\begin{example}[ASPIC, \cite{ModPra14,Pra10}]
\label{xmpl:aspic}
In ASPIC we work with a formal language \(\mathcal{L}\), a contrariness function \(\overline{\cdot}: \mathcal{L} \rightarrow \wp(\mathcal{L})\), a set of defeasible rules \(\mathcal{D}\) and a set of strict rules \(\mathcal{R}\) of the form \(A_1, \dotsc, A_n \Rightarrow A\) resp.\ \(A_1, \dotsc, A_n \rightarrow A\). Similarly as was the case for ABA, the strict rules may reflect domain-specific knowledge or be generated in view of an underlying core logic $\mathsf{L}$. We assume that \(\mathcal{L}\) contains for each defeasible rule \(R \in \mathcal{D}\) a logical atom \(n(R)\) that serves as name of \(R\). An \((\mathcal{D},\mathcal{R})\)-deduction of \(\phi \in \mathcal{L}\) from \(\Delta \subseteq \mathcal{L}\) is given by a tree
  
  
\begin{itemize}
\item whose leaves are labeled by elements in \(\Delta\) (so that each \(\delta \in \Delta\) occurs as label of a leaf),
\item for every non-root node labeled by \(\psi\) there is a rule  \(R = \phi_1, \dotsc, \phi_n \rightarrow \psi \in \mathcal{R}\) or \(R = \phi_1, \dotsc, \phi_n \Rightarrow \psi \in \mathcal{D}\) and its child-nodes are labeled by \(\phi_1, \dotsc, \phi_n\) (if \(R\) has an empty body, the single child-node is unlabeled). The edges connecting the child-nodes with the parent are labeled \(R\).\footnote{Usually edges are not labeled with rules in ASPIC (and so in cases of rules with empty bodies, there are usually no child-nodes either). We introduce these labels since they enable us to define our representation in a simpler way. 
  We also simplify the presentation in that we do not assume there to be defeasible premises.}
\item the root of the tree is labeled by \(\phi\).
\end{itemize}

Given a \((\mathcal{D},\mathcal{R})\)-derivation \(a\), \(\mathsf{DefC}(a)\) [\(\mathsf{StrC}(a)\)] is the set of all node labels to which an edge labeled with a defeasible [strict] rule leads and \(\mathsf{DefR}(a)\) [$\mathsf{StrR}(a)$] is the set of all edge labels that are defeasible [strict] rules.

An argumentation theory is a triple \((\mathcal{P}, \mathcal{R}, \mathcal{D})\) where \(\mathcal{P} \subseteq \mathcal{L}\) is a set of premises, \(\mathcal{R}\) is a set of strict rules and \(\mathcal{D}\) is a set of defeasible rules. The set \(\mathit{Arg}_{\sf aspic}(\mathcal{P},\mathcal{R},\mathcal{D})\) is the set of all \((\mathcal{D},\mathcal{R})\)-derivations of some \(\phi \in \mathcal{L}\) from some finite \(\Delta \subseteq \mathcal{P}\). Given two arguments \(a,b \in \mathit{Arg}_{\sf aspic}(\mathcal{P},\mathcal{R},\mathcal{D})\), \(a\) \emph{rebuts} \(b\) iff there is a \(\phi \in \mathsf{DefC}(b)\) such that \(\mathsf{Conc}(a) \in \overline{\phi}\); \(a\) \emph{undercuts} \(b\) iff \(a \in \overline{n(R)}\) for some \(R \in \mathsf{DefR}(b)\). Attack diagrams, underlying Dung-semantics $\sem$ and consequence relations \(\dblnc_{\star{\sf sem}}^{\sf aspic}\) for $\star\in\{\cap,\Cap,\cup\}$ are then defined in the usual way.

To represent ASPIC in our setting we first need to define our derivability relation and then translate the ASPIC attacks. In case the set of strict rules $\mathcal{R}$ presents domain-specific knowledge we define:
\begin{description}
\item[$(\dagger_{\sf aspic})$] \(\Gamma \vdash \phi\) iff there is a \((\mathcal{D},\mathcal{R})\)-derivation $a$ of \(\phi\) from \(\mathcal{P}\) where \(\Gamma = \{R, n(R) \mid R \in \mathsf{DefR}(a)\} \cup \mathsf{DefC}(a) \cup \mathsf{StrR}(a) \cup \{\rightarrow \psi \mid \psi \in \mathcal{P}\}\).\footnote{Similar as in the case of ABA we enrich the language $\mathcal{L}$ for $\vdash$ to track syntactic relevance. See Footnote~\ref{fn:prem-rules}.}
\end{description}
If $\mathcal{R}$ is generated via an underlying core logic we define:
\begin{description}
\item[$(\ddagger_{\sf aspic})$] \(\Gamma \vdash \phi\) iff there is a \((\mathcal{D},\mathcal{R})\)-derivation $a$ of \(\phi\) from \(\mathcal{P}\) where \(\Gamma = \{R, n(R) \mid R \in \mathsf{DefR}(a)\} \cup \mathsf{DefC}(a) \cup \{ \rightarrow \psi \mid \psi \in \mathcal{P}\}\).
\end{description}

For reasons of space we omit the proof that, where $\mathcal{S} = \{R,n(R), \mathsf{Conc}(R) \mid R \in \mathcal{D}\}  \cup \{ \rightarrow \psi \mid \psi \in \mathcal{P}\}$ and $\widehat{\cdot} = \mathsf{id}(\cdot)$,\footnote{For the variants ASPIC$^{-}$~\cite{CMO14} and ASPIC$^{\ominus}$~\cite{HeyStr17} where rebut is unrestricted we need to add $\mathsf{StrC}(a)$ to $\Gamma$ in ($\dagger_{\sf aspic}$) and ($\ddagger_{\sf aspic}$). For generalized rebut in ASPIC$^{\ominus}$ we can proceed analogous to Example~\ref{xmpl:CL:def}.} the setting $\mathcal{AF}_{\vdash}(\mathcal{S} \cup \mathcal{R})$ [respectively $\mathcal{AF}_{\vdash}(\mathcal{S})$] represents the ASPIC theory $(\mathcal{P},\mathcal{R},\mathcal{D})$ for $\vdash$ in ($\dagger_{\sf aspic}$) [respectively in ($\ddagger_{\sf aspic}$)], where $\star\in\{\cap,\Cap,\cup\}$ so that $(\mathcal{P}, \mathcal{R},\mathcal{D}) \:\dblnc_{\star{\sf sem}}^{\sf aspic}\: \phi$ iff $\mathcal{S} \cup \mathcal{R}\: \nc_{\star{\sf sem}}^{\mathcal{AF}_{\vdash}}\: \phi$ [respectively $\mathcal{S} \nc_{\star{\sf sem}}^{\mathcal{AF}_{\vdash}} \phi$].

\end{example}

\begin{remark}
  \label{rem:aspic:non-inter} 
  Analogous to Remark \ref{rem:aba:non-inter}, if ($\dagger$) holds, we obtain Non-Interference for the presentation $(\dagger_{\sf aspic})$ and for $(\ddagger_{\sf aspic})$ if additionally the underlying logic $\mathsf{L}$ satisfies Pre-Relevance. 
\end{remark}

\begin{remark}
  \label{rem:aspic:cum} 
  Our representation of ASPIC makes use of the pointed $\widehat{\cdot}$ (namely $\mathsf{id}$) and $(\mathcal{D},\mathcal{R})$-derivability satisfies Cut. Note that $\mathcal{AF}_{\vdash^{+\phi}}(\mathcal{S} \cup \mathcal{R})$ [respectively $\mathcal{AF}_{\vdash^{+\phi}}(\mathcal{S})$] adequately represents the ASPIC argumentation theory $(\mathcal{P} \cup \{\phi\}, \mathcal{R},\mathcal{D})$ for $\vdash$ in ($\dagger_{\sf aspic}$) [respectively for $\vdash$ in ($\ddagger_{\sf aspic}$)] and $\mathcal{S}$ as specified in Example~\ref{xmpl:aspic}.
  Thus we obtain Cumulativity for grounded semantics.
\end{remark}

%% file: meta-theory-flat.tex
\section{Meta-Theory}
\label{sec:org530aa41}

Now that we have shown how some of the best-known approaches to structured argumentation can be represented in the general framework from Section~\ref{sec:GeneralSetting}, we return to the meta-theory, introduced in Section~\ref{sec:RelevanceProperties}. First we show that non-interference (Definition~\ref{def:non-interference}) holds for argumentation frameworks that satisfy pre-relevance, Theorem~\ref{thm:non-inter}. This means that, under grounded, complete and preferred semantics, for (weakly) skeptical and credulous entailments, a consequence always shares some atomic formula with the premise set. Then we turn to the results for semantic relevance. Theorem~\ref{thm:cum}, based on a general setting, where we just suppose that $\vdash$ satisfies Cut and that $\widehat{\cdot}$ is pointed, shows (extensional) cumulativity for grounded semantics. With a few additional assumptions Theorem~\ref{thm:CUM:con} shows that we have (extensional) cumulativity for skeptical entailment. These two theorems show that, when information that could be derived previously is added to the given information, the conclusions do not change. In the last result, Theorem~\ref{thm:monotonicty}, we show that for credulous entailment we get something even stronger: monotonicity. Meaning that conclusions that could previously be derived, can still be derived when information is added.  

%
%
%

\subsection{Syntactic Relevance}
\label{sec:org31d680c}

In this section we prove Theorem~\ref{thm:non-inter}, concerning non-interference (see Definition~\ref{def:non-interference}). In the following we suppose that $\mathcal{AF}_{\vdash}$ is a setting that satisfies Pre-Relevance (see Definition~\ref{def:pre-rel-AF}). We start with some notations:

\begin{definition}
Where \(\mathcal{S} \subseteq \mathcal{L}\) and $a,b \in \mathit{Arg}_{\vdash}(\mathcal{S})$, we write $a \preceq b$ iff $\widehat{\mathsf{Supp}(a)} \subseteq \widehat{\mathsf{Supp}(b)}$.
\end{definition}

\begin{definition}
Where $\mathcal{S} \subseteq \mathcal{L}$ and $\mathcal{E} \subseteq \mathit{Arg}_{\vdash}(\mathcal{S})$, let 
$\mathsf{Defended}(\mathcal{E}, \mathcal{AF}_{\vdash}(\mathcal{S}))$ be the set of all arguments $a \in \mathit{Arg}_{\vdash}(\mathcal{S})$ that are defended by arguments in $\mathcal{E}$.
\end{definition}

\begin{definition}
  Let $\mathcal{A}^+$ denote the set of arguments attacked by the set of arguments $\mathcal{A}$.
\end{definition}
%
%


In view of the monotonicity of~~ $\widehat{\cdot}$~~ we have:
\begin{fact}
\label{fact:le:prec}
Where \(b' \preceq b\), if \(a\) attacks \(b'\) then \(a\) attacks \(b\).
\end{fact}

Complete extensions are closed under $\preceq$:
\begin{fact}
\label{fact:att:cmp}
Where \(\mathcal{S} \subseteq \mathcal{L}\), \(\mathcal{E} \in \exts_\cmp(\mathcal{AF}_{\vdash}(\mathcal{S}))\), \(a \in \mathcal{E}\), and \(b \in \mathit{Arg}_{\vdash}(\mathcal{S})\), then $b \in \mathcal{E}$ if $b \preceq a$. 
\end{fact}

\begin{lemma}
  \label{lem:arg:prerel}
  Where \(\mathcal{S} \mid \mathcal{S}'\), if $a \in \mathit{Arg}_{\vdash}(\mathcal{S} \cup \mathcal{S}')$ attacks $b \in \mathit{Arg}_{\vdash}(\mathcal{S})$, there is an $a' \in \mathit{Arg}_{\vdash}(\mathcal{S} \cap \mathsf{Supp}(a))$ that attacks $b$.
\end{lemma}
\begin{proof}
  Suppose $a = (\Gamma, \psi) \in \mathit{Arg}_{\vdash}(\mathcal{S} \cup \mathcal{S}')$ attacks $b = (\Lambda, \sigma) \in \mathit{Arg}_{\vdash}(\mathcal{S})$. Then, $\psi \in \overline{\phi}$ for some $\phi \in \widehat{\Lambda}$. Where $\mathcal{A}_1 = \mathsf{Atoms}(\mathcal{S})$, $\mathcal{A}_2 = \mathsf{Atoms}(\mathcal{S}')$, $\mathcal{T}_1 = \Lambda$, $\mathcal{T}_2 = \emptyset$, $\mathcal{S}_1 = \Gamma \cap \mathcal{S}$ and $\mathcal{S}_2 = \Gamma \cap \mathcal{S}'$, with Definition~\ref{def:prime:setting}, $\mathcal{S}_1' \vdash \psi'$ where $\mathcal{S}_1' \subseteq \mathcal{S}_1$,  $\psi' \in \overline{\phi'}$ and $\phi' \in \widehat{\Lambda}$. Thus, $(\mathcal{S}_1', \psi') \preceq a$ attacks $b$.
\end{proof}


\begin{lemma}
\label{lem:oplus:adm}
Where \(\mathcal{S} \mid \mathcal{S}'\), \(\mathcal{E} \in \exts_\cmp(\mathcal{AF}_{\vdash}(\mathcal{S}))\), \(\mathcal{E}' \in \exts_\cmp(\mathcal{AF}_{\vdash}(\mathcal{S}'))\), \(\mathcal{E} \cup \mathcal{E}' \in \exts_\adm(\mathcal{AF}_{\vdash}(\mathcal{S} \cup \mathcal{S}'))\).
\end{lemma}
\begin{proof}
Suppose \(\mathcal{S} \mid \mathcal{S}'\), \(\mathcal{E} \in \exts_\cmp(\mathcal{AF}_{\vdash}(\mathcal{S}))\) and \(\mathcal{E}' \in \exts_\cmp(\mathcal{AF}_{\vdash}(\mathcal{S}'))\). We now show that \(\mathcal{E} \cup \mathcal{E}'\) is admissible. 

\emph{Conflict-free}: Assume for a contradiction that there are \(a, a' \in \mathcal{E} \cup \mathcal{E}'\) such that \(a\) attacks \(a'\). By the conflict-freeness of \(\mathcal{E}\) and \(\mathcal{E}'\) it is not the case that \(a, a' \in \mathcal{E}\) or \(a, a' \in \mathcal{E}'\). Without loss of generality suppose \(a \in \mathcal{E}\) and \(a' \in \mathcal{E}'\). By Lemma~\ref{lem:arg:prerel}, there is a \(b \in \mathit{Arg}_{\vdash}(\mathcal{S}' \cap \mathsf{Supp}(a)) = \mathit{Arg}_{\vdash}(\emptyset)\) that attacks \(a'\). Thus, \(b\) is trivially defended by \(\mathcal{E}'\) and by the completeness of \(\mathcal{E}'\), \(b \in \mathcal{E}'\). This is a contradiction to the conflict-freeness of \(\mathcal{E}'\). %

\emph{Admissibility}: Suppose some \(b \in \mathit{Arg}_{\vdash}(\mathcal{S} \cup \mathcal{S}')\) attacks some \(a \in \mathcal{E} \cup \mathcal{E}'\). Without loss of generality assume \(a \in \mathcal{E}\). By Lemma \ref{lem:arg:prerel}, there is a \(b' \in \mathit{Arg}_{\vdash}(\mathcal{S} \cap \mathsf{Supp}(b))\) that attacks \(a\). Thus, there is a \(c \in \mathcal{E}\) that attacks \(b'\). By Fact~\ref{fact:le:prec}, \(c\) attacks \(b\). 
\end{proof}


\begin{lemma}
\label{lm:prime:att}
Where \(\mathcal{S}_1 \mid \mathcal{S}_2\), $a, b \in \mathit{Arg}_{\vdash}(\mathcal{S}_1 \cup \mathcal{S}_2)$, $\mathsf{Supp}(b) = \Theta$ and $b$ attacks $a$, 
\begin{enumerate}
\item some $b' \in \mathit{Arg}_{\vdash}(\mathcal{S}_1 \cap \Theta) \cup \mathit{Arg}_{\vdash}(\mathcal{S}_2 \cap \Theta)$ attacks \(a\);
\item if \(a \in \mathit{Arg}_{\vdash}(\mathcal{S}_1)\), some $b' \in \mathit{Arg}_{\vdash}(\mathcal{S}_1 \cap \Theta)$ attacks \(a\).
\end{enumerate}
\end{lemma}
\begin{proof}
  Let \(a = (\Gamma, \alpha) \in \mathit{Arg}_{\vdash}(\mathcal{S}_1 \cup \mathcal{S}_2)\). Suppose \(b = (\Theta, \beta)\) attacks \(a\). Thus, there is a \(\gamma \in \widehat{\Gamma}\) such that \(\beta \in \overline{\gamma}\). By Definition~\ref{def:pre-rel-AF}~(ii), there are \(i \in \{1,2\}\), \(\Theta' \subseteq \Theta \cap \mathcal{S}_i\), \(\phi \in \widehat{\Gamma \cap \mathcal{S}_i}\) and \(\psi \in \overline{\phi}\) such that \(b' = (\Theta', \psi) \in \mathit{Arg}_{\vdash}(\mathcal{S}_i)\). By Definition~\ref{def:pre-rel-AF}~(iii), $\widehat{\Gamma \cap \mathcal{S}_i} \subseteq \widehat{\Gamma}$ and hence $b'$ attacks \(a\). 
  
For Item 2 note that \(i = 1\) when setting $\mathcal{T}_1 = \Gamma$ and $\mathcal{T}_2 = \emptyset$ in Definition~\ref{def:prime:setting}. 
\end{proof}

\begin{lemma}
\label{lem:defdef}
Where \(\mathcal{S} \mid \mathcal{S}'\), \(\mathcal{E} \in \exts_\cmp(\mathcal{AF}_{\vdash}(\mathcal{S} \cup \mathcal{S}'))\), \(\mathcal{E}_1 = \mathcal{E} \cap \mathit{Arg}_{\vdash}(\mathcal{S})\) and \(\mathcal{E}_2 = \mathcal{E} \cap \mathit{Arg}_{\vdash}(\mathcal{S}')\), 
\begin{enumerate}
\item \(\mathcal{E} = \mathsf{Defended}(\mathcal{E}_1 \cup \mathcal{E}_{2}, \mathcal{AF}_{\vdash}(\mathcal{S} \cup \mathcal{S}'))\);
\item \(\mathcal{E}_1 \in \exts_\cmp(\mathcal{AF}_{\vdash}(\mathcal{S}))\). 
\end{enumerate}
\end{lemma}
\begin{proof}
  \emph{Ad 1.} Suppose \(\mathcal{E}\) defends some \(a \in \mathit{Arg}_{\vdash}(\mathcal{S} \cup \mathcal{S}')\). By Lemma~\ref{lm:prime:att} and Fact~\ref{fact:att:cmp}, $\mathcal{E}_1 \cup \mathcal{E}_2$ defends $a$.

  \emph{Ad 2.} 
  Note that \(\mathcal{E}_1\) is conflict-free since \(\mathcal{E}\) is conflict-free. Suppose \(b \in \mathit{Arg}_{\vdash}(\mathcal{S})\) attacks some \(a \in \mathcal{E}_1\). Thus, there is a \(c \in \mathcal{E}\) that attacks \(b\). By Lemma~\ref{lm:prime:att} and Fact~\ref{fact:att:cmp}, $\mathcal{E}_1$ attacks $b$. Thus, $\mathcal{E}_1$ is admissible. Suppose $\mathcal{E}_1$ defends some $d \in \mathit{Arg}_{\vdash}(\mathcal{S})$. Then $\mathcal{E}$ defends $d$ and hence $d \in \mathcal{E} \cap \mathit{Arg}_{\vdash}(\mathcal{S}) = \mathcal{E}_1$. Hence, $\mathcal{E}_1$ is complete.
\end{proof}




\begin{lemma}
\label{lem:cmp:oplus}
Where \(\mathcal{S} \mid \mathcal{S}'\), \(\mathcal{E}_1 \in  \exts_\cmp(\mathcal{AF}_{\vdash}(\mathcal{S}))\), \(\mathcal{E}_2 \in \exts_\cmp(\mathcal{AF}_{\vdash}(\mathcal{S}'))\), \(\mathcal{E} = \mathsf{Defended}(\mathcal{E}_1 \cup \mathcal{E}_2, \mathcal{AF}_{\vdash}(\mathcal{S} \cup \mathcal{S}'))\), 
\begin{enumerate}
\item \(\mathcal{E} \cap \mathit{Arg}_{\vdash}(\mathcal{S}) = \mathcal{E}_1\) and \(\mathcal{E} \cap \mathit{Arg}_{\vdash}(\mathcal{S}') = \mathcal{E}_2\).
\item \(\mathcal{E} \in \exts_\cmp(\mathcal{AF}_{\vdash}(\mathcal{S} \cup \mathcal{S}'))\).
\end{enumerate}
\end{lemma}
\begin{proof}
  \emph{Ad 1.} Suppose \(a \in \mathit{Arg}_{\vdash}(\mathcal{S}) \cap \mathcal{E}\). Thus, $a$ it is defended by \(\mathcal{E}_1 \cup \mathcal{E}_2\) in \(\mathcal{AF}_{\vdash}(\mathcal{S} \cup \mathcal{S}')\). Suppose some \(b \in \mathit{Arg}_{\vdash}(\mathcal{S})\) attacks \(a\). Thus, there is a \(c = (\Lambda, \sigma) \in \mathcal{E}_1 \cup \mathcal{E}_2\) that attacks \(b\). If $c \in \mathcal{E}_2$,  
  by Lemma \ref{lem:arg:prerel}, there is a \(c'  \in \mathit{Arg}_{\vdash}(\mathcal{S} \cap \Lambda) = \mathit{Arg}_{\vdash}(\emptyset)\) that attacks \(b\). Since \(c'\) has no attackers, by the completeness of \(\mathcal{E}_1\), \(c' \in \mathcal{E}_1\). Altogether this shows that \(a \in \mathsf{Defended}(\mathcal{E}_1, \mathcal{AF}_{\vdash}(\mathcal{S} \cup \mathcal{S}'))\). Again, by the completeness of \(\mathcal{E}_1\), \(a \in \mathcal{E}_1\). Thus, \(\mathcal{E} \cap \mathit{Arg}_{\vdash}(\mathcal{S}) = \mathcal{E}_1\). Analogously, \(\mathcal{E} \cap \mathit{Arg}_{\vdash}(\mathcal{S}') = \mathcal{E}_2\). This is Item 1. 
  
  \medskip

  \emph{Ad 2.} Suppose there are \(a,b \in \mathcal{E}\) such that \(a\) attacks \(b\). We know that there is a \(c \in \mathcal{E}_1 \cup \mathcal{E}_2\) that attacks \(a\). 
  Without loss of generality suppose \(c \in \mathcal{E}_1\). Thus, there is a \(d \in \mathcal{E}_1 \cup \mathcal{E}_2\) that attacks \(c\). Since by Lemma~\ref{lem:oplus:adm}, \(\mathcal{E}_1 \cup \mathcal{E}_2 \in \exts_\adm(\mathcal{AF}_{\vdash}(\mathcal{S} \cup \mathcal{S}'))\) we have reached a contradiction. Thus, \(\mathcal{E}\) is conflict-free.

  Suppose now some \(a \in \mathit{Arg}_{\vdash}(\mathcal{S} \cup \mathcal{S}')\) attacks some \(b \in \mathcal{E}\). By the definition of \(\mathcal{E}\) there is a \(c \in \mathcal{E}_1 \cup \mathcal{E}_2\) that attacks \(b\). By item 1, $c \in \mathcal{E}$. Thus, \(\mathcal{E}\) is admissible.

  For completeness assume that \(\mathcal{E}\) defends some \(a \in \mathit{Arg}_{\vdash}(\mathcal{S} \cup \mathcal{S}')\). Suppose \(b = (\Lambda, \beta) \in \mathit{Arg}_{\vdash}(\mathcal{S} \cup \mathcal{S}')\) attacks \(a\). Hence, there is a \(c \in \mathcal{E}\) that attacks \(b\). In view of Lemma~\ref{lm:prime:att} and Fact~\ref{fact:le:prec} there is a $c' \in (\mathcal{E} \cap \mathit{Arg}_{\vdash}(\mathcal{S})) \cup (\mathcal{E} \cap \mathit{Arg}_{\vdash}(\mathcal{S}'))$ that attacks $b$. By Item 1, $c' \in \mathcal{E}_1 \cup \mathcal{E}_2$ and therefore $a \in \mathsf{Defended}(\mathcal{E}_1 \cup \mathcal{E}_2, \mathcal{AF}_{\vdash}(\mathcal{S} \cup \mathcal{S}')) = \mathcal{E}$.
\end{proof}

%

\begin{lemma}
  \label{lem:prf:narrowing}
  Where \(\mathcal{S} \mid \mathcal{S}'\), \(\mathcal{E} \in \exts_\prf(\mathcal{AF}_{\vdash}(\mathcal{S} \cup \mathcal{S}'))\) and \(\mathcal{E}_1 = \mathcal{E} \cap \mathit{Arg}_{\vdash}(\mathcal{S})\), also \(\mathcal{E}_1 \in \exts_\prf(\mathcal{AF}_{\vdash}(\mathcal{S}))\).
\end{lemma}
\begin{proof} Let \(\mathcal{S} \mid \mathcal{S}'\), \(\mathcal{E} \in \exts_\prf(\mathcal{AF}_{\vdash}(\mathcal{S} \cup \mathcal{S}'))\), \(\mathcal{E}_1 = \mathcal{E} \cap \mathit{Arg}_{\vdash}(\mathcal{S})\), and \(\mathcal{E}_2 = \mathcal{E} \cap \mathit{Arg}_{\vdash}(\mathcal{S}')\). By Lemma~\ref{lem:defdef}, \(\mathcal{E}_1 \in \exts_\cmp(\mathcal{AF}_{\vdash}(\mathcal{S}))\) and \(\mathcal{E}_2 \in \exts_\cmp(\mathcal{AF}_{\vdash}(\mathcal{S}'))\). Suppose \(\mathcal{E}_1' \in \exts_\cmp(\mathcal{AF}_{\vdash}(\mathcal{S}))\) for which \(\mathcal{E}_1 \subseteq \mathcal{E}_1'\). By Lemma~\ref{lem:cmp:oplus}, where \(\mathcal{E}' = \mathsf{Defended}(\mathcal{E}_1' \cup \mathcal{E}_2, \mathcal{AF}_{\vdash}(\mathcal{S} \cup \mathcal{S}'))\), \(\mathcal{E}' \in \exts_\cmp(\mathcal{AF}_{\vdash}(\mathcal{S} \cup \mathcal{S}'))\). By Lemma~\ref{lem:defdef}, \(\mathcal{E} = \mathsf{Defended}(\mathcal{E}_1 \cup \mathcal{E}_2, \mathcal{AF}_{\vdash}(\mathcal{S} \cup \mathcal{S}'))\) and thus \(\mathcal{E} \subseteq \mathcal{E}'\). Since \(\mathcal{E}\) is preferred, \(\mathcal{E} = \mathcal{E}'\) and hence \(\mathcal{E}_1 = \mathcal{E}_1'\). Thus, \(\mathcal{E}_1\) is \(\subseteq\)-maximal and \(\mathcal{E}_1 \in \exts_\prf(\mathcal{AF}_{\vdash}(\mathcal{S}))\).
\end{proof}

\begin{lemma}
\label{lem:reduct:2}
Where \(\mathcal{S} \mid \mathcal{S}'\), \(\mathcal{E}_1 \in \exts_\sem(\mathcal{AF}_{\vdash}(\mathcal{S}))\), \(\sem\in\{\cmp,\prf,\grd\}\), there is a \(\mathcal{E} \in \exts_\sem(\mathcal{AF}_{\vdash}(\mathcal{S} \cup \mathcal{S}'))\) for which \(\mathcal{E}_1 = \mathcal{E} \cap \mathit{Arg}_{\vdash}(\mathcal{S})\).
\end{lemma}
\begin{proof}
(\(\mathsf{sem}= \mathsf{cmp}\)) Let \(\mathcal{E}_2\) be arbitrary in \(\exts_\cmp(\mathcal{AF}_{\vdash}(\mathcal{S}'))\). By Lemma~\ref{lem:cmp:oplus}, \(\mathcal{E} = \mathsf{Defended}(\mathcal{E}_1 \cup \mathcal{E}_2, \mathcal{AF}_{\vdash}(\mathcal{S} \cup \mathcal{S}')) \in \exts_\cmp(\mathcal{AF}_{\vdash}(\mathcal{S} \cup \mathcal{S}'))\) and \(\mathcal{E}_1 = \mathcal{E} \cap \mathit{Arg}_{\vdash}(\mathcal{S})\). 


\medskip

(\(\mathsf{sem} = \mathsf{grd}\)) Let \(\mathcal{E}_1 = \exts_\grd(\mathcal{AF}_{\vdash}(\mathcal{S}))\), \(\mathcal{E}_2 = \exts_\grd(\mathcal{AF}_{\vdash}(\mathcal{S}'))\). Again, by Lemma \ref{lem:cmp:oplus}, \(\mathcal{E} = \mathsf{Defended}(\mathcal{E}_1 \cup \mathcal{E}_2, \mathcal{AF}_{\vdash}(\mathcal{S} \cup \mathcal{S}')) \in \exts_\cmp(\mathcal{AF}_{\vdash}(\mathcal{S} \cup \mathcal{S}'))\), \(\mathcal{E}_1 = \mathcal{E} \cap \mathit{Arg}_{\vdash}(\mathcal{S})\), and \(\mathcal{E}_2 = \mathcal{E} \cap \mathit{Arg}_{\vdash}(\mathcal{S}')\). Suppose there is a \(\mathcal{E}^{\star} \subset \mathcal{E}\) such that \(\mathcal{E}^{\star} \in \exts_\cmp(\mathcal{AF}_{\vdash}(\mathcal{S} \cup \mathcal{S}'))\). By Lemma \ref{lem:defdef}, \(\mathcal{E}^{\star} \cap \mathit{Arg}_{\vdash}(\mathcal{S}) \in \exts_\cmp(\mathcal{AF}_{\vdash}(\mathcal{S}))\) and \(\mathcal{E}^{\star} \cap \mathit{Arg}_{\vdash}(\mathcal{S}') \in \exts_\cmp(\mathcal{AF}_{\vdash}(\mathcal{S}'))\). Thus, \(\mathcal{E}^{\star} \cap \mathit{Arg}_{\vdash}(\mathcal{S}) = \exts_\grd(\mathcal{AF}_{\vdash}(\mathcal{S}))\) and \(\mathcal{E}^{\star} \cap \mathit{Arg}_{\vdash}(\mathcal{S}') = \exts_\grd(\mathcal{AF}_{\vdash}(\mathcal{S}'))\). However, by Lemma \ref{lem:defdef}, \(\mathcal{E}^{\star} = \mathsf{Defended}(\mathcal{E}_1 \cup \mathcal{E}_2, \mathcal{AF}_{\vdash}(\mathcal{S}\cup \mathcal{S}')) = \mathcal{E}\), a contradiction with the assumption that \(\mathcal{E}^{\star}\subset\mathcal{E}\). 

\medskip

($\mathsf{sem} = \mathsf{prf}$) Let $\mathcal{E}_2 \in \exts_\prf(\mathcal{AF}_{\vdash}(\mathcal{S}'))$ be arbitrary. By Lemma~\ref{lem:cmp:oplus} it is known that $\mathcal{E} = \mathsf{Defended}(\mathcal{E}_1 \cup \mathcal{E}_2, \mathcal{AF}_{\vdash}(\mathcal{S} \cup \mathcal{S}')) \in \exts_\cmp(\mathcal{AF}_{\vdash}(\mathcal{S} \cup \mathcal{S}'))$. Assume for a contradiction that there is a $\mathcal{E}' \in \exts_\prf(\mathcal{AF}_{\vdash}(\mathcal{S} \cup \mathcal{S}'))$ such that $\mathcal{E}' \supset \mathcal{E}$. Let $\mathcal{E}_1' = \mathcal{E}' \cap \mathit{Arg}_{\vdash}(\mathcal{S})$ and $\mathcal{E}_2' = \mathcal{E}' \cap \mathit{Arg}_{\vdash}(\mathcal{S}')$. By Lemma~\ref{lem:defdef}, $\mathcal{E}' = \mathsf{Defended}(\mathcal{E}_1' \cup \mathcal{E}_2', \mathcal{AF}_{\vdash}(\mathcal{S} \cup \mathcal{S}'))$. Thus, $\mathcal{E}_1 \subset \mathcal{E}_1'$ or $\mathcal{E}_2 \subset \mathcal{E}_2'$. However, since by Lemma~\ref{lem:defdef}, $\mathcal{E}_1' \in \exts_\cmp(\mathcal{AF}_{\vdash}(\mathcal{S}))$ and $\mathcal{E}_2' \in \exts_\cmp(\mathcal{AF}_{\vdash}(\mathcal{S}'))$ this is a contradiction to $\mathcal{E}_1$ and $\mathcal{E}_2$ being preferred.
\end{proof}

\noninter*

\begin{proof}
Suppose \(\mathcal{S} \cup \{\phi\} \mid \mathcal{S}'\). To avoid clutter we will omit $\calAF_\vdash$ from the notation. We consider each of the consequence relations from Definition~\ref{definition:nc} for $\sem\in\{\grd,\cmp,\prf\}$:

\begin{itemize}
  \item \emph{Skeptical consequence:} We start by showing that $\calS\:\nc_\sem^\cap\:\phi$ iff $\calS\cup\calS'\:\nc_\sem^\cap\:\phi$. 
  
  ($\Ra$) Suppose that $\calS\:\nc^\cap_\sem\:\phi$. Then there is an argument $(\Gamma,\phi)$ such that $(\Gamma,\phi)\in\bigcap\exts_\sem(\calAF_\vdash(\calS))$. Let $\ext\in\exts_\sem(\calAF_\vdash(\calS\cup\calS'))$. By Lemmas~\ref{lem:defdef} and~\ref{lem:prf:narrowing} it follows that $\ext\cap\textit{Arg}_\vdash(\calS)\in\exts_\sem(\calAF_\vdash(\calS))$. Thus $(\Gamma,\phi)\in\bigcap\exts_\sem(\calAF_\vdash(\calS\cup\calS'))$ and hence $\calS\cup\calS'\:\nc^\cap_\sem\:\phi$. 
  
  ($\Leftarrow$) Suppose that $\calS\cup\calS'\:\nc^\cap_\sem\:\phi$. Thus there is a $(\Gamma,\phi)\in\bigcap\exts_\sem(\calAF_\vdash(\calS\cup\calS'))$, hence $\Gamma\vdash\phi$. By the Pre-Relevance of $\vdash$, since $\calS\cup\{\phi\}\mid\calS'$, there is some $\Gamma'\subseteq\calS$ such that $\Gamma'\vdash\phi$ as well. Note that $(\Gamma',\phi)\in\mathit{Arg}_\vdash(\calS)$. Let $\ext_1\in\exts_\sem(\calAF_\vdash(\calS))$, by Lemma~\ref{lem:reduct:2}, there is an $\ext\in\exts_\sem(\calAF_\vdash(\calS\cup\calS'))$ such that $\ext_1 = \ext\cap\textit{Arg}_\vdash(\calS)$. Therefore $(\Gamma',\phi)\in\bigcap\exts_\sem(\calAF_\vdash(\calS))$. Thus $\calS\:\nc^\cap_\sem\:\phi$.
  
  \item \emph{Weakly skeptical consequence:} We show that \(\mathcal{S} \nc_{{\sf sem}}^\Cap \phi\) iff \(\mathcal{S} \cup \mathcal{S}' \:\nc_{{\sf sem}}^\Cap\: \phi\):
  
  (\(\Rightarrow\)) Suppose \(\mathcal{S} \cup \mathcal{S}'~ {\not\!\!\nc_{{\sf sem}}^\Cap}\: \phi\). Thus, there is an \(\mathcal{E} \in \exts_\sem(\mathcal{AF}_{\vdash}(\mathcal{S} \cup \mathcal{S}'))\) for which there is no \(a \in \mathcal{E}\) with conclusion \(\phi\). By Lemmas~\ref{lem:defdef} and \ref{lem:prf:narrowing}, \(\mathcal{E}_1 = \mathcal{E} \cap \mathit{Arg}_{\vdash}(\mathcal{S}) \in \exts_\sem(\mathcal{AF}_{\vdash}(\mathcal{S}))\). Since there is no \(a \in \mathcal{E}_1\) with conclusion \(\phi\), \(\mathcal{S} ~ {\not\!\!\nc_{{\sf sem}}^\Cap}~ \phi\). 

(\(\Leftarrow\)) Suppose \(\mathcal{S} ~ {\not\!\!\nc_{{\sf sem}}^\Cap} ~ \phi\). Thus, there is an \(\mathcal{E} \in \exts_\sem(\mathcal{AF}_{\vdash}(\mathcal{S}))\) for which there is no \(a \in \mathcal{E}\) with conclusion \(\phi\). By Lemma \ref{lem:reduct:2}, there is an \(\mathcal{E}' \in \exts_\sem(\mathcal{AF}_{\vdash}(\mathcal{S} \cup \mathcal{S}'))\) for which \(\mathcal{E}' \cap \mathit{Arg}_{\vdash}(\mathcal{S}) = \mathcal{E}\). Assume for a contradiction that there is an argument \(a \in \mathcal{E}'\) with $\mathsf{Conc}(a) = \phi$. By the Pre-Relevance of $\vdash$, there is an \(a' =(\Gamma', \phi) \in \mathit{Arg}_{\vdash}(\mathcal{S} \cap \mathsf{Supp}(a))\). By Fact \ref{fact:att:cmp}, \(a' \in \mathcal{E}\) which contradicts our main supposition. Thus, \(\mathcal{S} \cup \mathcal{S}' \not\!\!\nc_{{\sf sem}}^\Cap \phi\). 

  \item \emph{Credulous consequences:} We show that \(\mathcal{S} \:\nc_{\mathsf{sem}}^\cup\: \phi\) iff \(\mathcal{S} \cup \mathcal{S}' \:\nc_{{\sf sem}}^\cup\: \phi\).
  
  (\(\Rightarrow\)) Suppose that $\calS\:\nc_{{\sf sem}}^\cup\:\phi$. Then there is some $\ext\in\exts_\sem(\AF_\vdash(\calS))$ such that there is an $a\in\ext$ with $\conc(a) = \phi$. By Lemma~\ref{lem:reduct:2}, there is an $\ext'\in\exts_\sem(\AF_\vdash(\calS\cup\calS'))$ for which $\ext = \ext'\cap\text{Arg}_\vdash(\calS)$. Thus $a\in\ext'$, from which it follows that $\phi\in{\sf Concs}(\ext')$ and thus $\calS\cup\calS'\:\nc_{{\sf sem}}^\cup\:\phi$.
  
  (\(\Leftarrow\)) Now assume that $\calS\cup\calS'\:\nc_{{\sf sem}}^\cup\:\phi$. Then there is some $\ext\in\exts_\sem(\AF_\vdash(\calS\cup\calS'))$ such that there is an $a\in\ext$ with $\conc(a) =\phi$. By the Pre-Relenvace of $\vdash$, there is an $a' = (\Gamma',\phi)\in\text{Arg}_{\vdash}(\calS\cap{\sf Supp}(a))$. By Fact~\ref{fact:att:cmp}, \(a' \in \mathcal{E}\). By Lemmas~\ref{lem:defdef} and \ref{lem:prf:narrowing}, \(\mathcal{E}' = \mathcal{E} \cap \mathit{Arg}_{\vdash}(\mathcal{S}) \in \exts_\sem(\mathcal{AF}_{\vdash}(\mathcal{S}))\) and thus $a'\in\ext'$. It follows that there is an $\ext^*\in\exts_\sem(\AF_\vdash(\calS))$ such that $a'\in\ext^*$ and hence $\phi\in{\sf Concs}(\ext)$. Therefore $\calS\:\nc_{{\sf sem}}^\cup\:\phi$. \qedhere
\end{itemize}
\end{proof}

\subsection{Semantic Relevance}
\label{sec:org67d6cb1}

We now turn to semantic relevance, concerning Cumulativity and Extensional Cumulativity for grounded semantics (Theorem~\ref{thm:cum}) and for the more general setting where $\sem = \{\grd,\prf,\stb\}$ (Theorem~\ref{thm:CUM:con}). Note that Theorem \ref{thm:cum} is a direct consequence of Theorem \ref{thm:grd:cum} below.


\begin{remark}
  The grounded extension can also be characterized  inductively by: \(\exts_\grd(\mathcal{AF}_{\vdash}(\mathcal{S})) = \bigcup_{\alpha \ge 0} \exts_\grd^{\alpha}(\mathcal{AF}_{\vdash}(\mathcal{S}))\) such that for \(\alpha = 0\): \(\exts_\grd^0(\mathcal{AF}_{\vdash}(\mathcal{S})) = \mathsf{Defended}(\emptyset, \mathcal{AF}_{\vdash}(\mathcal{S}))\), for successor ordinals $\alpha +1$ we have that: \(\exts_\grd^{\alpha+1}(\mathcal{AF}_{\vdash}(\mathcal{S})) = \mathsf{Defended}(\exts_\grd^{\alpha}(\mathcal{AF}_{\vdash}(\mathcal{S})), \mathcal{AF}_{\vdash}(\mathcal{S}))\), and for limit ordinals $\beta$ the characterization is defined by: $\exts_\grd^{\beta}(\mathcal{AF}_{\vdash}(\mathcal{S})) = \mathsf{Defended}(\bigcup_{\alpha < \beta}\exts_\grd^{\alpha}(\mathcal{AF}_{\vdash}(\mathcal{S})), \mathcal{AF}_{\vdash}(\mathcal{S}))$.
\end{remark}




\begin{theorem}
  \label{thm:grd:cum}
  Where $\mathcal{AF}_{\vdash} = (\vdash, \overline{\cdot}, \hat{\cdot})$, \(\vdash\) satisfies Cut and \(\widehat{\cdot}\) is pointed, if \(\mathcal{S} \nc_{\sf grd}^{\mathcal{AF}_{\vdash}} \phi\) then
\begin{enumerate}
\item there is a $(\Phi,\phi) \in \exts_\grd(\mathcal{AF}_{\vdash}(\mathcal{S}))$,
\item \(\exts_\grd(\mathcal{AF}_{\vdash}(\mathcal{S})) \subseteq \exts_\grd(\mathcal{AF}_{\vdash^{+\phi}}(\mathcal{S}))\),
\item \(\exts_\grd(\mathcal{AF}_{\vdash^{+\phi}}(\mathcal{S})) \cap \mathit{Arg}_{\vdash}(\mathcal{S}) = \exts_\grd(\mathcal{AF}_{\vdash}(\mathcal{S}))\),
\item for every \(a = (\Gamma,\gamma) \in \exts_\grd(\mathcal{AF}_{\vdash^{+\phi}}(\mathcal{S})) \setminus \mathit{Arg}_{\vdash}(\mathcal{S})\), \((\Gamma \cup \Phi, \gamma) \in \exts_\grd(\mathcal{AF}_{\vdash}(\mathcal{S}))\).
\end{enumerate}
\end{theorem}
\begin{proof}
  \emph{Ad 1.} This is due to the fact that $\mathcal{S} \nc_{\sf grd}^{\mathcal{AF}_{\vdash}} \phi$.
  
  \emph{Ad 2.} We give an inductive proof. 
  
  (Base) Let \(a \in \exts_\grd^0(\mathcal{AF}_{\vdash}(\mathcal{S}))\). Suppose some \(b = (\Gamma, \gamma) \in \mathit{Arg}_{\vdash^{+\phi}}(\mathcal{S})\) attacks \(a\). Thus, \(b \notin \mathit{Arg}_{\vdash}(\mathcal{S})\). Thus, \(b' = (\Gamma \cup \Phi, \gamma) \in \mathit{Arg}_{\vdash}(\mathcal{S})\) by Cut. This is a contradiction since \(b'\) attacks \(a\). So $a$ has no attackers in $\mathit{Arg}_{\vdash^{+\phi}}$ and so $a \in \exts_\grd(\mathcal{AF}_{\vdash^{+\phi}}(\mathcal{S}))$.

  (Step) We consider a successor ordinal $\alpha + 1$. 
  Let \(a \in \exts_\grd^{\alpha+1}(\mathcal{AF}_{\vdash}(\mathcal{S}))\). Suppose \(b = (\Gamma, \gamma) \in \mathit{Arg}_{\vdash^{+\phi}}(\mathcal{S})\) attacks \(a\). If \(b \in \mathit{Arg}_{\vdash}(\mathcal{S})\) there is a \(c \in \exts_\grd^{\alpha}(\mathcal{AF}_{\vdash}(\mathcal{S}))\) that attacks \(b\). By the inductive hypothesis (IH), \(c \in \exts_\grd(\mathcal{AF}_{\vdash^{+\phi}}(\mathcal{S}))\). Otherwise, by Cut \(b' = (\Gamma \cup \Phi,\gamma) \in \mathit{Arg}_{\vdash}(\mathcal{S})\) and $b'$ attacks \(a\). Thus, there is a \(d \in \exts_\grd^\alpha(\mathcal{AF}_{\vdash}(\mathcal{S}))\) that attacks \(b'\) in some \(\beta \in \widehat{\Gamma \cup \Phi}\). Since \(~\widehat{\cdot}~\) is pointed, \(\beta \in \widehat{\Gamma} \cup \widehat{\Phi}\). Since \((\Phi,\phi) \in \exts_\grd(\mathcal{AF}_{\vdash}(\mathcal{S}))\), \(\beta \in \widehat{\Gamma}\) and hence \(d\) attacks \(b\) in \(\mathcal{AF}_{\vdash^{+\phi}}(\mathcal{S})\). By IH, $d \in \exts_\grd(\mathcal{AF}_{\vdash^{+\phi}}(\mathcal{S}))$. Altogether this shows that $a$ is defended by $\exts_\grd(\mathcal{AF}_{\vdash^{+\phi}}(\mathcal{S}))$ and thus $a \in \exts_\grd(\mathcal{AF}_{\vdash^{+\phi}}(\mathcal{S}))$.

  The case for limit ordinals $\alpha'$ is analogous.


  \smallbreak

  \emph{Ad 3 and 4.} We show both simultaneously via induction. 
  
  (Base) Let \(a = (\Gamma, \gamma) \in  \exts_\grd^0(\mathcal{AF}_{\vdash^{+\phi}}(\mathcal{S}))\). Suppose first that \(a \in \mathit{Arg}_{\vdash}(\mathcal{S})\). Since \(\mathit{Arg}_{\vdash}(\mathcal{S}) \subseteq \mathit{Arg}_{\vdash^{+\phi}}(\mathcal{S})\), there are no attackers of \(a\) in \(\mathit{Arg}_{\vdash}(\mathcal{S})\) and hence \(a \in \exts_\grd^0(\mathcal{AF}_{\vdash}(\mathcal{S}))\). Suppose now that \(a \notin \mathit{Arg}_{\vdash}(\mathcal{S})\). By Cut, \(a' = (\Gamma \cup \Phi, \gamma) \in \mathit{Arg}_{\vdash}(\mathcal{S})\). Suppose some \(b \in \mathit{Arg}_{\vdash}(\mathcal{S})\) attacks \(a'\) in some \(\beta \in \widehat{\Gamma \cup \Phi}\). By the pointedness of~~\(\widehat{\cdot}\), \(\beta \in \widehat{\Gamma} \cup \widehat{\Phi}\). Note that \(\beta \notin \widehat{\Gamma}\) since otherwise \(b\) attacks \(a\) but \(a\) has no attackers. Thus, \(\beta \in \widehat{\Phi}\). Hence, \(b\) attacks \((\Phi,\phi)\) and is thus attacked by \(\exts_\grd(\mathcal{AF}_{\vdash}(\mathcal{S}))\). Thus, $a'$ is defended by $\exts_\grd(\mathcal{AF}_{\vdash}(\mathcal{S}))$ and so $a' \in \exts_\grd(\mathcal{AF}_{\vdash}(\mathcal{S}))$.

  (Step) We consider a successor ordinal $\alpha + 1$. Let \(a = (\Gamma, \gamma) \in \exts_\grd^{\alpha+1}(\mathcal{AF}_{\vdash_{+\phi}}(\mathcal{S}))\). Suppose first that \(a \in \mathit{Arg}_{\vdash}(\mathcal{S})\). Suppose some \(b \in \mathit{Arg}_{\vdash}(\mathcal{S})\) attacks \(a\). Thus, there is a \(c = (\Lambda, \sigma)\in \exts_\grd^\alpha(\mathcal{AF}_{\vdash^{+\phi}}(\mathcal{S}))\) that attacks \(b\). By the inductive hypothesis, if \(c \in \mathit{Arg}_{\vdash}(\mathcal{S})\), \(c \in \exts_\grd(\mathcal{AF}_{\vdash}(\mathcal{S}))\) and otherwise \((\Lambda \cup \Phi, \sigma) \in \exts_\grd(\mathcal{AF}_{\vdash}(\mathcal{S}))\). In either case \(\exts_\grd(\mathcal{AF}_{\vdash}(\mathcal{S}))\) defends \(a\) from the attacker and thus \(a \in \exts_\grd(\mathcal{AF}_{\vdash}(\mathcal{S}))\). Suppose now that \(a \notin \mathit{Arg}_{\vdash}(\mathcal{S})\). By Cut, \(a' = (\Gamma \cup \Phi, \gamma) \in \mathit{Arg}_{\vdash}(\mathcal{S})\). Suppose some \(b \in \mathit{Arg}_{\vdash}(\mathcal{S})\) attacks \(a'\) in some \(\beta \in \widehat{\Gamma \cup \Phi}\). By the pointedness of \(\widehat{\cdot}\), \(\beta \in \widehat{\Gamma} \cup \widehat{\Phi}\). If \(\beta \in \widehat{\Phi}\), \(b\) attacks \((\Phi,\phi)\) and is thus attacked by \(\exts_\grd(\mathcal{AF}_{\vdash}(\mathcal{S}))\). If \(\beta \in \widehat{\Gamma}\), \(b\) attacks \(a\). Thus, there is a \(c = (\Lambda, \sigma)\in \exts_\grd^\alpha(\mathcal{AF}_{\vdash^{+\phi}}(\mathcal{S}))\) that attacks \(b\). By the inductive hypothesis, if \(c \in \mathit{Arg}_{\vdash}(\mathcal{S})\), \(c \in \exts_\grd(\mathcal{AF}_{\vdash}(\mathcal{S}))\) and otherwise \((\Lambda \cup \Phi, \sigma) \in \exts_\grd(\mathcal{AF}_{\vdash}(\mathcal{S}))\). In either case \(\exts_\grd(\mathcal{AF}_{\vdash}(\mathcal{S}))\) defends \(a\) from the attacker and thus \(a \in \exts_\grd(\mathcal{AF}_{\vdash}(\mathcal{S}))\).

  The case for limit ordinals $\alpha'$ is analogous.
\end{proof}

Our previous result does not generalize to preferred semantics or to $\vdash$ that do not satisfy Cut. We give two examples.
\begin{example}[\cite{Mak03}]
\label{xmpl:ASPIC:CUM:1}
Consider an ASPIC framework with defeasible rules \(\mathcal{D} = \{ n_0: \top \Rightarrow p; \:\: n_1: p \vee q \Rightarrow \neg p\}\), facts $\mathcal{P} = \emptyset$, the strict rules induced by classical logic (see Example~\ref{xmpl:aspic}), and $\overline{\phi} = \psi$ if $\phi = \neg \psi$ and $\overline{\phi} = \neg \phi$ else. Consider the ASPIC-arguments \(\mathsf{a}_0 = \top \Rightarrow p\); \(\mathsf{a} = \mathsf{a}_0 \rightarrow (p \vee q)\) and \(\mathsf{b} = \mathsf{a} \Rightarrow \neg p\). With \((\ddagger_{\mathsf{aspic}})\) we have the arguments \(a_0 = (\{n_0, \top \Rightarrow p, p\}, p)\), \(a = (\{n_0, \top \Rightarrow p, p\}, p \vee q)\) and \(b = (\{n_0, n_1, \top \Rightarrow p, p,  p \vee q \Rightarrow \neg p, \neg p\}, \neg p)\) in \(\mathcal{AF}_{\vdash}(\mathcal{S})\) where \(\mathcal{S} = \{ n_0, \top \Rightarrow p, n_1, p \vee q \Rightarrow \neg p, p, \neg p\}\). Note that \(b\) attacks \(a_0, a\) and \(b\) while \(a_0\) attacks \(b\). Thus, the only preferred extension contains both \(a_0\) and \(a\) which means that \(\mathcal{S}\: \nc_{\cap\mathsf{prf}}\: p\) and \(\mathcal{S}\: \nc_{\cap\mathsf{prf}}\: p \vee q\). Once we move to \(\vdash^{+(p \vee q)}\) we also have the argument \(c = (\{n_1, p \vee q \Rightarrow \neg p, \neg p\}, \neg p)\) attacking \(a_0\). It is easy to see that now \(\mathcal{S} ~ {\not\!\!\nc_{\cap\mathsf{prf}}^{+(p \vee q)}} ~ p\).
\end{example}

\begin{example}
  We now consider the same example but with ${\vdash_{\star}} = \left\{ (\Gamma,\phi) \in {\vdash} \mid \Gamma\allowbreak\mbox{ is} \right. \allowbreak\left.\mathsf{CL}\mbox{-consistent} \right\}$ and grounded extension. Unlike Example~\ref{xmpl:ASPIC:CUM:1}, \(b\) is not anymore  in \(\mathcal{AF}_{\vdash_{\star}}(\mathcal{S})\). Thus, \(\mathcal{S}\: \nc_{\mathsf{grd}}^{\mathcal{AF}_{\vdash_{\star}}} \:p\) and \(\mathcal{S}\: \nc_{\mathsf{grd}}^{\mathcal{AF}_{\vdash_{\star}}}\: p \vee q\). Once we move to \(\vdash_{\star}^{+(p \vee q)}\), \(c = (\{n_{1}, p \vee q \Rightarrow \neg p, \neg p\}, \neg p)\) again attacks \(a_0\) and \(a\) and thus, \(\mathcal{S} ~{\not\!\!\nc_{\mathsf{grd}}^{\mathcal{AF}^{+(p \vee q)}_{\vdash_{\star}}}}~ p\). Note that $\vdash_{\star}$ does not satisfy Cut.
\end{example}


We now turn to the proof of Theorem~\ref{thm:CUM:con}. 

\begin{definition}
Let \(\mathcal{AF}_{\vdash} = \langle \vdash, \overline{\cdot}, \mathsf{id} \rangle\). A set \(\Theta \subseteq \mathcal{S}\) is 
\emph{maximal \(\mathcal{AF}_{\vdash}(\mathcal{S})\)-consistent} iff there is no \(\mathcal{AF}_{\vdash}(\mathcal{S})\)-consistent \(\Theta' \subseteq \mathcal{S}\) such \(\Theta \subset \Theta'\). We write \(\mathsf{CS}(\mathcal{AF}_{\vdash}(\mathcal{S}))\) [\(\mathsf{MCS}(\mathcal{AF}_{\vdash}(\mathcal{S}))\)] for all [maximal] \(\mathcal{AF}_{\vdash}(\mathcal{S})\)-consistent sets.
\end{definition}

%


\begin{lemma}
\label{lm:MCS:1}
Where \(\mathcal{S} \subseteq \mathcal{L}\), \(\mathcal{AF}_{\vdash} = (\vdash, \overline{\cdot}, \mathsf{id})\) is contrapositable, \(\vdash\) satisfies Cut, and \(\Theta \in \mathsf{MCS}(\mathcal{AF_{\vdash}}(\mathcal{S}))\), \(\mathit{Arg}_{\vdash}(\Theta) \in \exts_\sem(\mathcal{AF}_{\mathsf{con}}(\mathcal{S}))\).
\end{lemma}
\begin{proof}
Suppose \(\Theta \in \mathsf{MCS}(\mathcal{AF}_{\vdash}(\mathcal{S}))\). To show that \(\mathcal{E} = \mathit{Arg}_{\vdash}(\Theta)\) is conflict-free assume there are \(a,b \in \mathcal{E}\) such that \(a = (\Gamma, \gamma)\) attacks \(b = (\Lambda, \sigma)\) where \(\gamma \in \overline{\lambda}\) and \(\lambda \in \Lambda\). Then \(\Gamma \cup \{\lambda\} \subseteq \Theta\). If \(\lambda \in \Gamma\) then by contraposition, \(\Gamma \setminus \{\lambda\} \vdash \gamma'\) for some \(\gamma' \in \overline{\lambda}\) and hence \(\Theta \notin \mathsf{CS}(\mathcal{AF}_{\vdash}(\mathcal{S}))\). If \(\lambda \notin \Gamma\), also \(\Theta \notin \mathsf{CS}(\mathcal{AF}_{\vdash}(\mathcal{S}))\). We have reached a contradiction.

Consider some \(b = (\Lambda, \sigma) \in \mathit{Arg}_{\mathsf{con}}(\mathcal{S}) \setminus \mathcal{E}\). Then \(\Theta \cup \{\beta\}\) is \(\mathcal{AF}_{\vdash}(\mathcal{S})\)-inconsistent for some \(\beta \in \Lambda\). Thus, there is a \(\Theta' \subseteq \Theta\) for which \((\Theta' \cup \{\beta\}) \setminus \{\gamma\} \vdash \sigma'\) where \(\sigma' \in \overline{\gamma}\) for some \(\gamma \in \Theta \cup \{\beta\}\). If \(\gamma = \beta\), \(\Theta' \setminus \{\gamma\} \vdash \sigma'\) where \(\sigma' \in \overline{\beta}\). In this case let \(a' = (\Theta' \setminus \{\gamma\}, \sigma') \in \mathcal{E}\). If \(\gamma \neq \beta\), by contraposition \(\Theta' \cup \{\gamma\} \vdash \beta'\) for some \(\beta' \in \overline{\beta}\). In this case let \(a' = (\Theta' \cup \{\gamma\}, \beta') \in \mathcal{E}\). Since in any case \(a'\) attacks \(b\), \(\mathcal{E}\) is stable.
\end{proof}

\begin{lemma}
\label{lm:MCS:2}
Where \(\mathcal{AF}_{\vdash} = (\vdash, \overline{\cdot}, \mathsf{id})\) is contrapositable, \(\vdash\) satisfies Cut, and \(\mathcal{E} \in \exts_\prf(\mathcal{AF}_{\mathsf{con}}(\mathcal{S}))\), \(\bigcup \{ \mathsf{Supp}(a) \mid a \in \mathcal{E}\} \in \mathsf{MCS}(\mathcal{AF}_{\vdash}(\mathcal{S}))\).
\end{lemma}
\begin{proof}
Let \(\mathcal{E} \in \exts_\prf(\mathcal{AF}_{\mathsf{con}}(\mathcal{S}))\). Suppose \(\Lambda = \bigcup\{\mathsf{Supp}(a) \mid a \in \mathcal{E}\} \notin \mathsf{CS}(\mathcal{AF}_{\vdash}(\mathcal{S}))\). Thus, there is a \(\subseteq\)-minimal \(\Theta \subseteq \Lambda\) such that \(\Theta \setminus \{\gamma\} \vdash \psi\) where \(\psi \in \overline{\gamma}\) and \(\gamma \in \Theta\). Thus, \(\Theta \setminus \{\gamma\} \in \mathsf{CS}(\mathcal{AF}_{\vdash}(\mathcal{S}))\) and hence \(b = (\Theta \setminus \{\gamma\}, \psi) \in \mathit{Arg}_{\mathsf{con}}(\mathcal{S})\). Also, since \(\Theta \subseteq \Lambda\), there is a \(c \in \mathcal{E}\) for which \(\gamma \in \mathsf{Supp}(c)\). Then, \(b\) attacks \(c\). Hence, there is a \(d = (\Delta, \delta') \in \mathcal{E}\) that attacks \(b\) such that \(\delta' \in \overline{\delta}\) for some \(\delta \in \Theta \setminus \{\gamma\}\). Since \(\Theta \subseteq \Lambda\), there is an \(e \in \mathcal{E}\) with \(\delta \in \mathsf{Supp}(e)\). Thus, \(d\) attacks \(e\) which contradicts the conflict-freeness of \(\mathcal{E}\). Thus, \(\Lambda \in \mathsf{CS}(\mathcal{AF}_{\vdash}(\mathcal{S}))\).

Since \(\Lambda \in \mathsf{CS}(\mathcal{AF}_{\vdash}(\mathcal{S}))\), there is a \(\Theta \in \mathsf{MCS}(\mathcal{AF}_{\vdash}(\mathcal{S}))\) for which \(\Lambda  \subseteq \Theta\). By Lemma \ref{lm:MCS:1}, \(\mathit{Arg}_{\vdash}(\Theta) \in \exts_\stb(\mathcal{AF}_{\mathsf{con}}(\mathcal{S}))\) and hence \(\mathit{Arg}_{\vdash}(\Theta) = \mathcal{E}\) by the \(\subseteq\)-max\-i\-mal\-i\-ty of \(\mathcal{E}\). Thus, \(\Lambda = \Theta \in \mathsf{MCS}(\mathcal{AF}_{\vdash}(\mathcal{S}))\).
\end{proof}

\begin{corollary}
\label{cor:stb:prf:contra}
For a set of formulas \(\mathcal{S} \subseteq \mathcal{L}\) and where \(\mathcal{AF}_{\vdash} = (\vdash, \overline{\cdot}, \mathsf{id})\) is contrapositable and \(\vdash\) satisfies Cut, \(\exts_\stb(\mathcal{AF}_{\mathsf{con}}(\mathcal{S})) = \exts_\prf(\mathcal{AF}_{\mathsf{con}}(\mathcal{S}))\).
\end{corollary}
\begin{proof}
Suppose \(\mathcal{E} \in \exts_\prf(\mathcal{AF}_{\mathsf{con}}(\mathcal{S}))\). By Lemma \ref{lm:MCS:2}, \(\Theta = \bigcup \left\{ \mathsf{Supp}(a) \mid a \in \mathcal{E} \right\} \in \mathsf{MCS}(\mathcal{AF}_{\vdash}(\mathcal{S}))\). By Lemma \ref{lm:MCS:1}, \(\mathit{Arg}_{\vdash}(\Theta) \in \exts_\stb(\mathcal{AF}_{\mathsf{con}}(\mathcal{S}))\). By the \(\subseteq\)-maximality of \(\mathcal{E}\), \(\mathcal{E} = \mathit{Arg}_{\vdash}(\Theta)\) and hence \(\mathcal{E} \in \exts_\stb(\mathcal{AF}_{\mathsf{con}}(\mathcal{S}))\). It is well-known that every stable extension is preferred.
\end{proof}
%

\begin{lemma}
\label{lm:MCS:cum}
Where \(\mathcal{S} \subseteq \mathcal{L}\), \(\mathcal{AF}_{\vdash} = (\vdash, \overline{\cdot}, \mathsf{id})\) is contrapositable, \(\vdash\) satisfies Cut, and for every \(\Theta \in \mathsf{MCS}(\mathcal{AF}_{\vdash}(\mathcal{S}))\) there is a \(\Theta' \subseteq \Theta\) for which \(\Theta' \vdash \phi\), we have: \(\mathsf{MCS}(\mathcal{AF}_{\vdash}(\mathcal{S})) = \mathsf{MCS}(\mathcal{AF}_{\vdash}^{+\phi}(\mathcal{S}))\).
\end{lemma}
\begin{proof}
It suffices to show \(\mathsf{MCS}(\mathcal{AF}_{\vdash}(\mathcal{S})) \subseteq \mathsf{CS}(\mathcal{AF}_{\vdash}^{+\phi}(\mathcal{S}))\) and \(\mathsf{MCS}(\mathcal{AF}_{\vdash}^{+\phi}(\mathcal{S})) \subseteq \mathsf{CS}(\mathcal{AF}_{\vdash}(\mathcal{S}))\).

Let \(\Theta \in \mathsf{MCS}(\mathcal{AF}_{\vdash}^{+\phi}(\mathcal{S}))\). Assume \(\Theta \notin \mathsf{CS}(\mathcal{AF}_{\vdash}(\mathcal{S}))\). Thus, there is a \(\Theta' \subseteq \Theta\) and a \(\gamma \in \Theta'\) for which \(\Theta' \setminus \{\gamma\} \vdash \psi\) where \(\psi \in \overline{\gamma}\). But then \(\Theta' \setminus \{\gamma\} \vdash^{+\phi} \psi\) and hence \(\Theta \notin \mathsf{CS}(\mathcal{AF}_{\vdash}^{+\phi}(\mathcal{S}))\) which is a contradiction. Hence \(\Theta \in \mathsf{CS}(\mathcal{AF}_{\vdash}(\mathcal{S}))\).

Let \(\Theta \in \mathsf{MCS}(\mathcal{AF}_{\vdash}(\mathcal{S}))\). Assume \(\Theta \notin \mathsf{CS}(\mathcal{AF}_{\vdash}^{+\phi}(\mathcal{S}))\). Then there is a \(\Theta' \subseteq \Theta\) such that \(\Theta' \setminus \{\gamma\} \vdash^{+\phi} \psi\) where \(\psi \in \overline{\gamma}\) and \(\gamma \in \Theta'\). Since \(\Theta \in \mathsf{CS}(\mathcal{AF}_{\vdash}(\mathcal{S}))\), \(\Theta' \setminus \{\gamma\} \nvdash \psi\). Since \(\Theta'' \vdash \phi\) for some \(\Theta'' \subseteq \Theta\), by Cut, \(\Theta'' \cup (\Theta' \setminus \{\gamma\}) \vdash \psi\) and by contraposition \((\Theta'' \cup \Theta') \setminus \{\gamma\} \vdash \psi'\) for some \(\psi' \in \overline{\gamma}\). Thus, \(\Theta \notin \mathsf{CS}(\mathcal{AF}_{\vdash}(\mathcal{S}))\) which is a contradiction. Hence \(\Theta \in \mathsf{CS}(\mathcal{AF}_{\vdash}^{+\phi}(\mathcal{S}))\).
\end{proof}

\begin{lemma}
\label{lm:grd:free}
For a set of formulas \(\mathcal{S} \subseteq \mathcal{L}\) and where \(\mathcal{AF}_{\vdash}\) is contrapositable and \(\vdash\) satisfies Cut, we have that \(\exts_\grd(\mathcal{AF}_{\mathsf{con}}(\mathcal{S})) = \mathit{Arg}_{\vdash}\left( \bigcap \mathsf{MCS}(\mathcal{AF}_{\vdash}(\mathcal{S})) \right)\).
\end{lemma}
\begin{proof}
Suppose \(a \in \exts_\grd(\mathcal{AF}_{\mathsf{con}}(\mathcal{S}))\). Then \(a \in \bigcap \exts_\prf(\mathcal{AF}_{\mathsf{con}}(\mathcal{S}))\). By Lemmas \ref{lm:MCS:1}, \ref{lm:MCS:2}, and Corollary \ref{cor:stb:prf:contra}, \(a \in \bigcap_{\Theta \in \mathsf{MCS}(\mathcal{AF}_{\vdash}(\mathcal{S}))} \mathit{Arg}_{\vdash}(\Theta)\) and hence \(a \in \mathit{Arg}_{\vdash}\left( \bigcap \mathsf{MCS}(\mathcal{AF}_{\vdash}(\mathcal{S})) \right)\).

Suppose \(\Theta \subseteq \bigcap \mathsf{MCS}(\mathcal{AF}_{\vdash}(\mathcal{S}))\) and \(\Gamma \vdash \gamma\). Consider \(a = (\Gamma, \gamma) \in \mathit{Arg}_{\vdash}(\Theta)\). Suppose some \(b = (\Lambda, \sigma) \in \mathit{Arg}_{\mathsf{con}}(\mathcal{S})\) attacks \(a\). Thus, \(\sigma \in \overline{\alpha}\) for some \(\alpha \in \Gamma\). But then \(\Lambda \cup \{\alpha\} \notin \mathsf{CS}(\mathcal{AF}_{\vdash}(\mathcal{S}))\) while \(\Lambda \in \mathsf{CS}(\mathcal{AF}_{\vdash}(\mathcal{S}))\). Thus, there is a \(\Lambda' \in \mathsf{MCS}(\mathcal{AF}_{\vdash}(\mathcal{S}))\) for which \(\Lambda \subseteq \Lambda'\) and \(\alpha \notin \Lambda'\) which contradicts that \(\alpha \in \bigcap \mathsf{MCS}(\mathcal{AF}_{\vdash}(\mathcal{S}))\). So, \(a \in \exts_\grd(\mathcal{AF}_{\mathsf{con}}(\mathcal{S}))\) since it has no attackers.
\end{proof}

\begin{corollary}
\label{cor:free:CUM}
Where \(\mathcal{AF}_{\vdash}\) is a contrapositable argumentation setting for which \(\vdash\) satisfies Cut, if there is a \(\Theta \subseteq \bigcap \mathsf{MCS}(\mathcal{AF}_{\vdash}(\mathcal{S}))\) with \(\Theta \vdash \phi\) then \(\bigcap \mathsf{MCS}(\mathcal{AF}_{\vdash}(\mathcal{S})) = \bigcap \mathsf{MCS}(\mathcal{AF}_{\vdash}^{+\phi}(\mathcal{S}))\) for every \(\mathcal{S} \subseteq \mathcal{L}\).
\end{corollary}
\begin{proof}
This follows directly with Lemma \ref{lm:MCS:cum}.
\end{proof}

\begin{fact} \label{fact:vdash:plus}
If \(\vdash\) satisfies Cut, \(\vdash^{+\phi}\) satisfies Cut. If \(\mathcal{AF}_{\vdash}\) is contrapositable, \(\mathcal{AF}_{\vdash^{+\phi}}\) is contrapositable.
\end{fact}
\begin{proof}
Suppose \(\Gamma \vdash^{+\phi} \psi\) and \(\Gamma', \psi \vdash^{+\phi} \psi'\). Then \(\Gamma \vdash \psi\) or \(\Gamma, \phi \vdash \psi\) and \(\Gamma', \psi \vdash \psi'\) or \(\Gamma',\psi, \phi \vdash \psi'\). Since \(\vdash\) satisfies Cut we get \(\Gamma, \Gamma' \vdash \psi'\) or \(\Gamma, \Gamma', \phi \vdash \psi'\). Hence, \(\Gamma, \Gamma' \vdash^{+\phi} \psi'\). 

Suppose \(\Theta \vdash^{+\phi} \gamma'\) for some \(\gamma' \in \overline{\gamma}\) and let \(\sigma \in \Theta\). Then either \(\Theta \vdash \gamma'\) or \(\Theta, \phi \vdash \gamma'\). Since \(\mathcal{AF}_{\vdash}\) is contrapositable, for some \(\sigma' \in \overline{\sigma}\), either \((\Theta \cup \{\gamma\}) \setminus \{\sigma\} \vdash \sigma'\) or \((\Theta \cup \{\gamma\}) \setminus \{\sigma\}, \phi \vdash \sigma'\). Thus, \((\Theta \cup \{\gamma\}) \setminus \{\sigma\} \vdash^{+\phi} \sigma'\).
\end{proof}

\CUMcon*

\begin{proof}
Let $\sem\in\{\prf,\stb\}$, $\star\in\{\cap,\Cap\}$ and $\mathcal{S} \subseteq \mathcal{L}$. Suppose $\mathcal{S} \nc_{\star\mathsf{sem}}^{\mathcal{AF}_{\mathsf{con}}} \phi$. Let $\Theta \in \mathsf{MCS}(\mathcal{AF}_{\vdash}(\mathcal{S}))$. By Lemma \ref{lm:MCS:1} and Corollary \ref{cor:stb:prf:contra}, $\mathit{Arg}_{\vdash}(\Theta) \in \exts_\sem(\mathcal{AF}_{\mathsf{con}}(\mathcal{S}))$. Since $\mathcal{S} \nc_{\star\mathsf{sem}}^{\mathcal{AF}_{\mathsf{con}}} \phi$, there is a $(\Gamma,\phi) \in \mathit{Arg}_{\vdash}(\Theta)$. Hence $\Gamma \vdash \phi$ and $\Gamma \subseteq \Theta$. Hence, by Lemma \ref{lm:MCS:cum}, ($\star$) $\mathsf{MCS}(\mathcal{AF}_{\vdash}(\mathcal{S})) = \mathsf{MCS}(\mathcal{AF}_{\vdash}^{+\phi}(\mathcal{S}))$.

\medskip

\emph{Extensional cumulativity.}
Suppose first that $\mathcal{E} \in \exts_\sem(\mathcal{AF}_{\mathsf{con}}(\mathcal{S}))$. Thus, by Lemma~\ref{lm:MCS:2} and Corollary~\ref{cor:stb:prf:contra}, $\bigcup \{\mathsf{Supp}(a) \mid a \in \mathcal{E} \} \in \mathsf{MCS}(\mathcal{AF}_{\vdash}(\mathcal{S}))$. By Lemma~\ref{lm:MCS:1} and Corollary~\ref{cor:stb:prf:contra}, $\mathit{Arg}_{\vdash}(\bigcup \{ \mathsf{Supp}(a) \mid a \in \mathcal{E}\}) \in \exts_\sem(\mathcal{AF}_{\mathsf{con}}(\mathcal{S}))$. Since we have that $\mathit{Arg}_{\vdash}(\bigcup \{ \mathsf{Supp}(a) \mid a \in \mathcal{E}\}) \supseteq \mathcal{E}$ and $\mathcal{E} \in \mathsf{Sem}(\mathcal{AF}_{\mathsf{con}}(\mathcal{S}))$, by the \(\subseteq\)-maximality of $\mathcal{E}$, $\mathcal{E} = \mathit{Arg}_{\vdash}(\bigcup \{ \mathsf{Supp}(a) \mid a \in \mathcal{E}\})$. Also, by ($\star$), $\bigcup \{ \mathsf{Supp}(a) \mid a \in \mathcal{E}\} \in \mathsf{MCS}(\mathcal{AF}_{\vdash}^{+\phi}(\mathcal{S}))$. Let $\mathcal{E}^+ = \mathit{Arg}_{\vdash}^{+\phi}(\bigcup \{ \mathsf{Supp}(a) \mid a \in \mathcal{E}\})$. By Fact~\ref{fact:vdash:plus}, Lemma~\ref{lm:MCS:1} and Corollary~\ref{cor:stb:prf:contra}, $\mathcal{E}^+ \in \exts_\sem(\mathcal{AF}_{\vdash}^{+\phi}(\mathcal{S}))$. Note that $\mathcal{E}^+ \cap \mathit{Arg}_{\vdash}(\mathcal{S}) = \mathcal{E}$. 

Suppose now that $\mathcal{E}^+ \in \exts_\sem(\mathcal{AF}_{\vdash}^{+\phi}(\mathcal{S}))$. By Lemma~\ref{lm:MCS:2} and Corollary~\ref{cor:stb:prf:contra}, $\bigcup \{ \mathsf{Supp}(a) \mid a \in \mathcal{E}^+\} \in \mathsf{MCS}(\mathcal{AF}_{\vdash}^{+\phi}(\mathcal{S}))$. By ($\star$) it follows that $\bigcup \{ \mathsf{Supp}(a) \mid a \in \mathcal{E}^+\} \in \mathsf{MCS}(\mathcal{AF}_{\vdash}(\calS))\) By Lemma~\ref{lm:MCS:1} and Corollary~\ref{cor:stb:prf:contra} for $\mathcal{E} = \mathit{Arg}_{\vdash}(\bigcup \{ \mathsf{Supp}(a) \mid a \in \mathcal{E}^+\})$, $\mathcal{E} \in \exts_\sem(\mathcal{AF}_{\vdash}(\mathcal{S}))$. Note that $\mathcal{E} = \mathcal{E}^+ \cap \mathit{Arg}_{\vdash}(\mathcal{S})$.

Altogether we have shown extensional cumulativity.
  
  
  \medskip
  
  \emph{Cumulativity.} We first consider $\star = \cap$. Suppose $\mathcal{S}\: \nc_{\cap\sem}^{\mathcal{AF}_{\sf con}} \:\psi$. Thus, there is a $(\Gamma, \psi) \in \bigcap \mathsf{Ext}_{\mathsf{sem}} (\mathcal{AF}_{\sf con}(\mathcal{S}))$. Let $\mathcal{E}' \in \mathsf{Ext}_{\mathsf{sem}}(\mathcal{AF}^{+\phi}_{\sf con}(\mathcal{S}))$. By extensional cumulativity it is known that $\mathcal{E}' \cap \mathrm{Arg}(\mathcal{AF}_{\sf con}(\mathcal{S})) \in \mathsf{Ext}_{\mathsf{sem}}(\mathcal{AF}_{\sf con}(\mathcal{S}))$. It this follows that $(\Gamma,\psi) \in \bigcap \mathsf{Ext}_{\mathsf{sem}}(\mathcal{AF}^{+\phi}_{\sf con}(\mathcal{S}))$ and thus $\mathcal{S}\: \nc_{\cap\sem}^{\mathcal{AF}^{+\phi}_{\sf con}} \:\psi$.

Suppose $\mathcal{S}\: \nc_{\cap\sem}^{\mathcal{AF}^{+\phi}_{\sf con}}\: \psi$. Thus, there is a $(\Gamma,\psi) \in \bigcap \mathsf{Ext}_{\mathsf{sem}}(\mathcal{AF}^{+\phi}_{\sf con}(\mathcal{S}))$. Let now $\mathcal{E}$ be arbitrary in $\mathsf{Ext}_{\mathsf{sem}}(\mathcal{AF}_{\sf con}(\mathcal{S}))$. By extensional cumulativity, there is an $\mathcal{E}' \in \mathsf{Ext}_{\mathsf{sem}}(\calAF_{\sf con}^{+\phi}(\mathcal{S}))$ for which $\mathcal{E} = \mathcal{E}' \cap \mathrm{Arg}(\mathcal{AF}_{\sf con}(\mathcal{S}))$. If $\phi \notin \Gamma$, $(\Gamma, \psi) \in \mathcal{E}$ and thus $(\Gamma, \psi ) \in \bigcap \mathsf{Ext}_{\mathsf{sem}}(\mathcal{AF}_{\sf con}(\mathcal{S}))$ since $\mathcal{E}$ was an arbitrary member of $\mathsf{Ext}_{\mathsf{sem}}(\mathcal{AF}_{\sf con}(\mathcal{S}))$. Thus, $\mathcal{S} \:\nc_{\cap\sem}^{\mathcal{AF}_{\sf con}}\: \psi$. Else, note first that since $\mathcal{S}\: \nc_{\cap\sem}^{\mathcal{AF}_{\sf con}} \:\phi$, there is an argument $(\Delta, \phi) \in \bigcap \mathsf{Ext}_{\mathsf{sem}}(\mathcal{AF}_{\sf con}(\mathcal{S}))$. So $(\Delta,\phi) \in \mathcal{E}$ and hence $(\Delta,\phi) \in \mathcal{E}'$. By Cut, $(\Delta \cup \Gamma, \psi) \in \textit{Arg}_\vdash(\calS)$. By Lemma \ref{lm:MCS:2}, $\bigcup \{ \mathsf{Supp}(a) \mid a \in \mathcal{E}'\} \in \mathsf{MCS}(\mathcal{AF}^{+\phi}_{\sf con}(\mathcal{S}))$. By Lemma \ref{lm:MCS:1}, $\textit{Arg}_\vdash(\bigcup \{\mathsf{Supp}(a) \mid a \in \mathcal{E}'\}) \in \mathsf{Ext}_{\mathsf{sem}}(\mathcal{AF}^{+\phi}_{\sf con}(\mathcal{S}))$. By the \(\subseteq\)-maximality of $\mathcal{E}'$, $\mathcal{E}' = \bigcup \{\mathsf{Supp}(a) \mid a \in \mathcal{E}'\}$ and so $(\Delta \cup \Gamma,\psi) \in \mathcal{E}'$. Since $\Delta \cup \Gamma \subseteq \mathcal{S}$ also $(\Delta \cup \Gamma, \psi) \in \mathcal{E}$. Thus, $(\Delta \cup \Gamma, \psi) \in \bigcap \mathsf{Ext}_{\mathsf{sem}}(\mathcal{AF}_{\sf con}(\calS))$ since $\mathcal{E}$ was an arbitrary member of $\mathsf{Ext}_{\mathsf{sem}}(\mathcal{AF}_{\sf con}(\mathcal{S}))$. Hence, $\mathcal{S} \:\nc_{\cap\sem}^{\mathcal{AF}_{\sf con}}\: \psi$.

\medskip
  
  Now let $\star = \Cap$. Suppose \(\mathcal{S} ~{\nnc_{\Cap\mathsf{sem}}^{\mathcal{AF}_{\mathsf{con}}}}~ \psi\).  Thus, there is an \(\mathcal{E} \in \exts_\sem(\mathcal{AF}_{\mathsf{con}}(\mathcal{S}))\) such that there is no \((\Gamma, \psi) \in \mathcal{E}\). By Lemma~\ref{lm:MCS:2} and Corollary~\ref{cor:stb:prf:contra}, $\bigcup\{ \mathsf{Supp}(a) \mid a \in \mathcal{E}\} \in \mathsf{MCS}(\mathcal{AF}_{\sf con}(\mathcal{S}))$. By Lemma~\ref{lm:MCS:1} and Corollary~\ref{cor:stb:prf:contra}, $\mathit{Arg}_{\vdash}(\bigcup \{ \mathsf{Supp}(a) \mid a \in \mathcal{E}\}) \in \exts_\sem(\mathcal{AF}_{\mathsf{con}}(\mathcal{S}))$ and thus, by the $\subseteq$-maximality of $\mathcal{E}$, $\mathcal{E} = \mathit{Arg}_{\vdash}(\bigcup \{ \mathsf{Supp}(a) \mid a \in \mathcal{E}\})$. Also, since \(\mathcal{S} \:\nc_{\Cap\mathsf{sem}}^{\mathcal{AF}_{\mathsf{con}}}\: \phi\) there is a \((\Delta, \phi) \in \mathcal{E}\). By extensional cumulativity, there is a \(\mathcal{E}' \in \exts_\sem(\mathcal{AF}_{\mathsf{con}}^{+\phi}(\mathcal{S}))\) such that \(\mathcal{E} = \mathcal{E}' \cap \mathit{Arg}_\vdash(\mathcal{S})\). Assume there is a \((\Lambda,\psi) \in \mathcal{E}' \setminus \mathcal{E}\). Thus, $\Lambda \nvdash \psi$ but $\Lambda, \phi \vdash \psi$. Since $\Delta \vdash \phi$, by Cut, $\Lambda \cup \Delta \vdash \psi$. Since $\Lambda \cup \Delta \subseteq \bigcup \{ \mathsf{Supp}(a) \mid a \in \mathcal{E}\}$, $(\Lambda \cup \Delta, \psi) \in \mathcal{E}$, which is a contradiction. Since there is no argument with conclusion $\psi$ in $\mathcal{E}'$, $\mathcal{S} ~{\nnc_{{\Cap}\mathsf{sem}}^{\mathcal{AF}_{\mathsf{con}}^{+\phi}}}~ \psi$. 

Suppose now for the other direction that \(\mathcal{S} ~{\nnc_{{\cap}\mathsf{sem}}^{\mathcal{AF}_{\mathsf{con}}^{+\phi}}}~ \psi\). Thus, there is an \(\mathcal{E}' \in \exts_\sem(\mathcal{AF}_{\mathsf{con}}^{+\phi}(\mathcal{S}))\) such that there is no \((\Gamma, \psi) \in \mathcal{E}'\). Since, by extensional cumulativity, \(\mathcal{E} = \mathcal{E}' \cap \mathit{Arg}_\vdash(\mathcal{S}) \in \exts_\sem(\mathcal{AF}_{\mathsf{con}}(\mathcal{S}))\), also \(\mathcal{S} ~{\nnc_{{\Cap}\mathsf{sem}}^{\mathcal{AF}_{\mathsf{con}}}}~ \psi\). 

\medskip

  Let \(\mathsf{sem} = \mathsf{grd}\) and suppose \(\mathcal{S} \nc_{\mathsf{\grd}}^{\mathcal{AF}_{\mathsf{con}}} \phi\). By Fact~\ref{fact:vdash:plus}, Lemma \ref{lm:grd:free} and Corollary \ref{cor:free:CUM}, \(\exts_\grd\left( \mathcal{AF}_{\mathsf{con}}(\mathcal{S}) \right) = \mathit{Arg}_{\vdash}\left( \bigcap \mathsf{MCS}(\mathcal{AF}_{\vdash}(\mathcal{S})) \right) = \mathit{Arg}_{\vdash}\bigl( \bigcap \mathsf{MCS}(\mathcal{AF}_{\vdash}^{+\phi}(\mathcal{S})) \bigr) = \exts_\grd(\mathcal{AF}_{\mathsf{con}}^{+\phi}(\mathcal{S}))\). Cumulativity then follows by an analogous argument as in the case of \(\mathsf{sem} \in \{\mathsf{prf}, \mathsf{stb}\}\).
\end{proof}

\begin{fact}
  \label{fact:extendedMCSs}
  Let $\calS,\calS'\subseteq\mathcal{L}$ where $\calS\subseteq\calS'$, then for each $\mathcal{T}\in{\sf MCS}(\AF_\vdash(\calS))$ there is a $\mathcal{T}'\in{\sf MCS}(\AF_\vdash(\calS'))$ such that $\mathcal{T}\subseteq\mathcal{T}'$. 
\end{fact}
\begin{proof}
Let $\calS,\calS'\subseteq\mathcal{L}$ such that $\calS\subseteq\calS'$, and let $\mathcal{T}\in{\sf MCS}^\preceq(\calS)$. Note that $\mathcal{T}\in{\sf CS}(\AF_\vdash(\calS'))$. Thus, by definition of a maximally consistent subset, there is some $\mathcal{T}'\in{\sf MCS}(\AF_\vdash(\calS'))$ such that $\mathcal{T}\subseteq\mathcal{T}'$. 
\end{proof}

\Monotonicity*
\begin{proof}
  Let $\sem\in\{\prf,\stb\}$ and $\mathcal{S},\calS' \subseteq \mathcal{L}$ such that $\calS\subseteq\calS'$. Suppose that $\calS\:\nc^{\AF_{\sf con}}_{\cup{\sf sem}}\:\phi$. Thus, there is some $\ext\in\exts_\sem(\AF_{\sf con}(\calS))$ such that there is an $a\in\ext$ with $a=(\Gamma,\phi)$. By Lemma~\ref{lm:MCS:2} and Corollary~\ref{cor:stb:prf:contra}, $\mathcal{T}=\bigcup\{{\sf Supp}(b)\mid b\in\ext\}\in{\sf MCS}(\AF_{\sf con}(\calS))$. By Fact~\ref{fact:extendedMCSs}, there is some $\mathcal{T}'\in{\sf MCS}(\AF_{\sf con}(\calS'))$ such that $\mathcal{T}\subseteq\mathcal{T}'$. By Lemma~\ref{lm:MCS:1} and Corollary~\ref{cor:stb:prf:contra}, $\textit{Arg}_\vdash(\mathcal{T}')\in\exts_\sem(\AF_{\sf con}(\calS'))$. Thus $a\in\textit{Arg}_\vdash(\mathcal{T}')$, from which it follows that $\calS'\:\nc^{\AF_{\sf con}}_{\cup{\sf sem}}\:\phi$. 
\end{proof}
